%% file: iclr2025_conference.tex
\documentclass{article} 

\usepackage{iclr2025_conference,times}

\input{math_commands.tex}

\input{header/name}
\input{header/residual}

\usepackage{natbib}
\usepackage{wrapfig}
\usepackage{hyperref}
\usepackage{url}

\title{Divergence of Empirical Neural Tangent Kernel in Classification Problems}

\author{
Zixiong Yu$^{1,2,\star}$~~Songtao Tian$^{2,\dag}$~~Guhan Chen$^{3,\ddag}$ \\
$^1$Noah’s Ark Lab, Huawei Technologies Ltd., Shenzhen, China \\
$^2$Department of Mathematical Sciences, Tsinghua University, Beijing, China \\
$^3$Department of Statistics and Data Science, Tsinghua University, Beijing, China \\
$^\star$First author: \texttt{yuzx19@tsinghua.org.cn} \\
$^\dag$Co-first author: \texttt{tiansongtao.2020@tsinghua.org.cn} \\
$^\ddag$Corresponding author: \texttt{chen-gh23@mails.tsinghua.edu.cn} 
}

\iclrfinalcopy 
\begin{document}

\maketitle

\begin{abstract}

This paper demonstrates that in classification problems, fully connected neural networks (FCNs) and residual neural networks (ResNets) cannot be approximated by kernel logistic regression based on the Neural Tangent Kernel (NTK) under overtraining (i.e., when training time approaches infinity). Specifically, when using the cross-entropy loss, regardless of how large the network width is (as long as it is finite), the empirical NTK diverges from the NTK on the training samples as training time increases. To establish this result, we first demonstrate the strictly positive definiteness of the NTKs for multi-layer FCNs and ResNets. Then, we prove that during training, 
the neural network parameters diverge if the smallest eigenvalue of the empirical NTK matrix (Gram matrix) with respect to training samples is bounded below by a positive constant. This behavior contrasts sharply with the lazy training regime commonly observed in regression problems. Consequently, using a proof by contradiction, 
we show that the empirical NTK does not uniformly converge to the NTK across all times on the training samples as the network width increases. We validate our theoretical results through experiments on both synthetic data and the MNIST classification task. This finding implies that NTK theory is not applicable in this context, with significant theoretical implications for understanding neural networks in classification problems. 

\end{abstract}

\input{sec_introduction.tex}
\input{sec_preliminaries}

\input{sec_results}

\input{sec_proof_sketch}
\input{sec_experiments}
\input{sec_discussion}

\clearpage

\bibliography{iclr2025_conference}
\bibliographystyle{iclr2025_conference}

\appendix

\input{sec_proof}

\end{document}

%% file: math_commands.tex
\usepackage{amsmath,amsfonts,bm}

\def\eqref#1{equation~\ref{#1}}

\def\1{\bm{1}}

\DeclareMathAlphabet{\mathsfit}{\encodingdefault}{\sfdefault}{m}{sl}
\SetMathAlphabet{\mathsfit}{bold}{\encodingdefault}{\sfdefault}{bx}{n}

\newcommand{\R}{\mathbb{R}}



%% file: header/name.tex
\usepackage{amsmath}
\usepackage{amssymb}
\usepackage{mathtools}
\usepackage{amsthm}
\newcommand{\qand}{\quad \text{and} \quad}

\usepackage[utf8]{inputenc} 
\usepackage[T1]{fontenc}    
\usepackage[colorlinks=true, linkcolor=black, citecolor=black, urlcolor=black]{hyperref}
\usepackage{booktabs}       

\usepackage{amsmath, amsfonts, amsthm, amssymb}       
\usepackage{nicefrac}       
\usepackage{microtype}      

\newcommand{\x}{\bm{x}}
\newcommand{\y}{\bm{y}}

\newcommand{\X}{\bm{X}}
\newcommand{\mr}{\mathrm}
\newcommand{\mb}{\mathbb}
\newcommand{\mc}{\mathcal}

\newcommand{\FCNTK}{K^{\mathrm{FC}}_{\mathrm{NT}}}
\newcommand{\RESNTK}{K^{\mathrm{Res}}_{\mathrm{NT}}}
\newcommand{\diag}{\mathrm{diag}}

\newcommand*{\abs}[1]{\left\lvert #1 \right\rvert}
\newcommand*{\norm}[1]{\left\lVert #1 \right\rVert}

\usepackage{mleftright}
\newcommand*{\dd}{\mathrm{d}}

\usepackage[pdftex]{graphicx}

\def\R{\mathbb R}

\def\bbS{\mathbb{S}}

\def\X{\boldsymbol{X}}

\def\x{\boldsymbol{x}}

\def\S{\mathbb S}

\def\me{\mathrm{e}}
\usepackage{graphicx}
\usepackage{subfigure}
\usepackage{amsmath,color}
\usepackage{algorithmicx,algorithm}
\usepackage{amssymb}
\usepackage{amsfonts}
\usepackage{multirow}
\usepackage{amsthm}
\numberwithin{equation}{section}
\newtheorem{theorem}{Theorem}
\newtheorem{lemma}{Lemma}
\newtheorem{corollary}{Corollary}
\newtheorem{proposition}{Proposition}
\theoremstyle{definition}
\newtheorem{definition}{Definition}

\newtheorem{remark}{Remark}

\usepackage[colorlinks=true, linkcolor=black, citecolor=black, urlcolor=black]{hyperref}

\newcommand{\bk}[1]{\left[#1\right]}

\newcommand{\ag}[1]{\left\langle#1\right\rangle}

\newcommand{\mpt}[1]{\mleft(#1\mright)}
\newcommand{\mbk}[1]{\mleft[#1\mright]}
\newcommand{\mcl}[1]{\mleft\{#1\mright\}}

\usepackage{environ}
\NewEnviron{iproof}[1][Proof]{\begin{proof}[\indent\bfseries #1]\BODY\end{proof}}{}

\providecommand{\cref}{\prettyref}

 \usepackage[noend]{algpseudocode}

\usepackage{algorithmicx,algorithm}

%% file: header/residual.tex
\usepackage[utf8]{inputenc} 
\usepackage[T1]{fontenc}    
\usepackage{amsmath, amsfonts, amsthm, amssymb}       
\usepackage{nicefrac}       
\usepackage{microtype}      
\usepackage{xcolor}         

\usepackage{mleftright}

\usepackage[pdftex]{graphicx}

\def\R{\mathbb R}

\def\bbS{\mathbb{S}}

\def\X{\boldsymbol X}

\def\x{\boldsymbol x}

\def\S{\mathbb S}

\def\me{\mathrm{e}}
\usepackage{graphicx}
\usepackage{subfigure}
\usepackage{amsmath,color}
\usepackage{algorithmicx,algorithm}
\usepackage{amssymb}
\usepackage{amsfonts}
\usepackage{multirow}
\usepackage{amsthm}
\numberwithin{equation}{section}
\usepackage{hyperref}
\usepackage{cleveref}

%% file: sec_introduction.tex
\section{Introduction}

In recent years, neural networks have achieved remarkable performance in various fields, including computer vision \citep{lecun1998gradient,Wang_2017_CVPR}, natural language processing \citep{devlin2019bertpretrainingdeepbidirectional}, and generative models \citep{Karras_2019_CVPR,ho2020denoising}. 
In computer vision, \citet{krizhevsky2012imagenet} proposed AlexNet, which significantly outperformed traditional methods in the ImageNet competition by leveraging deep convolutional networks. Building on this, \citet{he2016deep} introduced ResNet, which addressed the degradation problem in training deep networks by incorporating residual blocks, significantly improving model performance. For natural language processing, \citet{vaswani2017attention} proposed the Transformer, which greatly enhanced sequence-to-sequence tasks through the self-attention mechanism. Subsequently, \citet{dosovitskiy2021imageworth16x16words} went a step further by applying the Transformer to computer vision, which not only transformed the field but also laid the foundation for the development of future multimodal models. In the domain of generative models, \citet{kingma2013auto} introduced the Variational Autoencoder (VAE), enabling efficient training of generative models. 

In addition to the aforementioned achievements, neural networks have also influenced other disciplines. Consequently, Geoffrey Hinton, known as the father of neural networks, was awarded the Nobel Prize in Physics in 2024, in addition to the Turing Award. However, despite the excellent performance of neural networks in practical applications, there is still no complete theoretical explanation for why neural networks perform so well. This stark contrast not only adds to the mystique of neural networks, but also attracts a large number of researchers to explore their mysteries.

The theoretical properties of networks have been a subject of study for a long time. In terms of the expressive power of networks, \cite{hornik1989multilayer,cybenko1989approximation} proposed the well-known universal approximation theorem, which demonstrates that neural networks with a sufficient number of parameters can approximate any continuous function. Recently, many works have continued to explore the expressive power of networks in broader and more structured contexts, such as \citet{cohen2016expressive,hanin2017approximating,lu2017expressive}, among others. However, these works focus solely on the expressive capabilities of networks and neglect their statistical properties, thus failing to fully explain the excellent performance of networks.

Regarding the statistical properties of networks, many studies have focused on their generalization ability. Within the static non-parametric regression framework, \citet{bauer2019deep,schmidt2020nonparametric} investigated the minimax optimality of networks when different Hölder function classes are used as the true target functions. At the same time, the dynamic training of networks is equally important. \cite{du2018gradient,allen2019convergence,chizat2019lazy} studied the performance of sufficiently wide network on the training set after gradient descent and stochastic gradient descent in regression problems. Building on this, \citet{jacot2018neural} explicitly proposed the concept of the Neural Tangent Kernel (NTK) to more precisely characterize the performance of sufficiently wide networks during gradient descent. 

In the framework of NTK theory, research can typically be divided into two steps: studying the convergence of the  empirical NTK to the NTK; and investigating the performance of the corresponding NTK regressor within the kernel regression framework. The former step establishes the validity of the NTK theory, while the latter elucidates the properties of the network. For instance, \cite{arora2019exact,lee2019wide,hu2021regularization} verified the empirical NTK convergence of various networks under the Mean Square Error (MSE) loss function. Furthermore, \citet{lai2023generalizationabilitywideneural,li2024eigenvalue,chen2024impacts} demonstrated the uniform convergence of the empirical NTK for fully connected networks trained with MSE. Additionally, numerous works have investigated the generalization ability of the NTK regressor in MSE regression problems \citep{lee2019wide,hu2021regularization,suh2021non,lai2023generalizationabilitywideneural,li2024eigenvalue, tian2024improve}. However, for networks trained using the cross-entropy loss function, which is common in classification problems, the properties of the empirical NTK remain an open question.

This article primarily investigates the convergence of the empirical NTK in classification problems. As mentioned above, the convergence of the empirical NTK has been thoroughly studied in regression problems with the MSE loss function. In classification problems, the cross-entropy function is more suitable for handling categorical label data and thus has a wider range of applications compared to MSE. However, the convergence of the empirical NTK during training with the cross-entropy loss function, which is necessary, however, has hardly been researched. According to our results, under the cross-entropy loss function, the empirical NTK of multi-layer fully connected neural network and residual neural network will no longer uniformly converge to the NTK but instead diverge as time approaches infinity. This implies that NTK theory is no longer applicable in this case. From this perspective, our work has significant theoretical implications.

\subsection{Related Work}

The concept of the NTK was first introduced by \cite{jacot2018neural} to approximate the training process of wide neural networks using kernel regression. Substantial research has focused on the convergence of the empirical NTK to the NTK in regression problems, with early works establishing pointwise convergence \citep{du2019gradient, allen2019convergence, arora2019exact}.
Building on this, \citet{lai2023generalizationabilitywideneural,li2024eigenvalue,lai2023generalizationabilitywideresidual} have proven the uniform convergence of the empirical NTK: as the network width approaches infinity, the empirical NTK uniformly converges to the NTK across all samples and at every time in the training process (see \cref{eq: convergence}). 
For classification problems, due to the complexity of the cross-entropy loss, many researchers have used MSE loss as a substitute \citep{vyas2022limitations,hron2020infinite,lai2024optimalitykernelclassifierssobolev}.

Additionally, some studies explore the generalization ability of neural networks in classification problems under the cross-entropy loss from various perspectives. Recent works, such as \cite{ji2019polylogarithmic, taheri2024generalization}, have studied the dynamics of neural networks when the network width $m \to \infty$ and the training time $T$ is fixed and finite. These studies demonstrate the generalization ability of neural networks under weak conditions. In contrast, our work examines the case where $m$ is fixed and $T \to \infty$. We show that in this setting, the distance between the empirical NTK and the NTK has a positive lower bound that is independent of $m$, which prevents convergence to kernel predictors and reveals distinct overfitting dynamics in classification tasks. Considering that the final predictions of classification neural networks require the application of a sigmoid or softmax function, there are also works that define the empirical NTK differently from the one presented in this paper, such as \citet{liu2020linearity}. However, since NTK is designed to approximate neural networks as corresponding kernel methods (see \cref{eq: f_gradient_flow}), the empirical NTK for classification problems presented in this paper is set to be consistent with that in regression problems. 




 \subsection{Main Content and Contributions}

In this article, we study the convergence of the empirical NTK for fully connected and residual neural networks in classification problems. We show that when using the cross-entropy loss function, the empirical NTK of networks cannot achieve uniform convergence over time on the training samples.

First, we investigate the strictly positive definiteness of the NTK for fully connected and residual networks. We consider networks defined on a general compact domain. Through appropriate transformations, we convert the NTK of fully connected and residual networks into recursive and arc-cosine kernel forms, respectively, and then prove their strictly positive definiteness. Subsequently, we present the core conclusion of this paper: we analyze the dynamic properties of the networks and use a proof by contradiction to demonstrate that under the cross-entropy loss function, the empirical NTK cannot uniformly converge to the NTK over time on the training samples.

To the best of our knowledge, our article is the first to theoretically prove the divergence of the empirical NTK in classification problems. This aligns with the experiment results and reflects the limitations of earlier theoretical work on related issues. Overall, our article demonstrates that using the NTK to study the properties of networks in classification problems is not a good choice.

The rest of this paper is organized as follows. In \cref{sec: Preliminaries}, we introduce notation conventions and basic settings for neural networks. An introduction to the neural tangent kernel (NTK) and its properties will be presented in \cref{sec: NTK}. Our main results are displayed in \cref{sec: Main results}, with a proof sketch provided in \cref{Proof Sketch}. Numerical experiments supporting our main conclusions are detailed in \cref{sec: Numerical Experiments}. Finally, in \cref{sec: Discussion}, we discuss our conclusions and outline potential future work. 

%% file: sec_preliminaries.tex
\section{Preliminaries}\label{sec: Preliminaries}

\subsection{Basic Settings and Notations}\label{subsec: Basic Settings}
In this paper, we consider the binary classification problem. Suppose that we have observed $n$ samples $\mathcal{D}_n=\{(\bm{x}_{i},y_{i}), i \in [n]\}$, where $[n]$ denotes the index set $\{1,2,\dots,n\}$. We assume the pairs of training data $(\x_i,y_i)$ comes from $\mathcal{X} \times \{0,1\}$, where $\mathcal{X}$ is a compact subset of $\mathbb{R}^d$. The label $y_i$ comes from $\{0,1\}$ with unknown distribution (we do not make requirements on the separability of data). 
We denote the sample matrix as $\X = (\x_1, \x_2, \dots, \x_n)^\top \in \mathbb{R}^{n \times d}$, and $\y = (y_1, y_2, \dots, y_n)^\top \in \mathbb{R}^n$. For a function $f(\cdot): \mathbb{R}^d \to \mathbb{R}$, we use $f(\X)$ to represent the entry-wise application of the function to each element in $\X$. That is, $f(\X) = (f(\x_1), f(\x_2), \dots, f(\x_n))^\top$. This means that the function $f$ is applied individually to each sample in the sample matrix. Finally, we use $\norm{~\cdot~}$ or $\norm{~\cdot~}_2$ to denote the Euclidean norm, and $\norm{~\cdot~}_{\mr{F}}$ to denote the Frobenius norm.


\subsection{Neural Network}
The main focus of this paper is on fully connected networks (FCNs) and residual networks (ResNets). As the most fundamental and simplest form of neural networks, FCNs are crucial for understanding more complex network architectures. ResNets were first introduced by \citet{he2016deep}. This groundbreaking achievement won the 2015 ImageNet Large Scale Visual Recognition Challenge (ILSVRC), drawing widespread attention and sparking enthusiastic discussions worldwide due to its outstanding performance and exceptional accuracy. 
The original ResNets appeared in the form of convolutional neural networks. However, in theoretical research, the focus is often on the impact of the residual structure itself on the network's performance. Therefore, it is usually simplified to a fully connected network with residual connections (abbreviated as ResFCN, with the network structure referenced in \citet{huang2020deep,belfer2024spectral}). Specifically, unless otherwise stated, the term "residual network (ResNet)" used later in this paper refers to ResFCN. 

\paragraph{Fully Connected Network (FCN)} We consider a fully connected network with $L$ hidden layers, each having a width of $m$. The specific expression is as follows:
\begin{align}\label{eq: fc_network}
\begin{split}
f(\x;\bm{\theta}) &= \bm{W}^{(L+1)}\bm{\alpha}^{(L)}(\x);\quad \bm{\alpha}^{(0)}(\x) = \x;\\
\bm{\alpha}^{(l)}(\x) &= \sqrt{\tfrac{2}{m}}\, \sigma\mpt{\bm{W}^{(l)} \bm{\alpha}^{(l-1)}(\x) + \bm{b}^{(l)}}, \quad l\in[L],
\end{split}
\end{align}
where $\sigma$ is the entry-wise ReLU function defined as $\sigma(x) \coloneqq \max\{0,x\}$, which acts component-wise on vectors. For $l \in [L]$, the matrices $\bm{W}^{(l)}$ and $\bm{b}^{(l)}$ have dimensions $\mathbb{R}^{m \times m}$ and $\mathbb{R}^{m \times 1}$, respectively. Additionally, the dimension of $\bm{W}^{(L+1)}$ is $\mathbb{R}^{1 \times m}$. We use $\bm{\theta}$ to denote the vector obtained by flattening the above parameters. All parameters are initialized as independent and identically distributed (i.i.d.) random variables from the standard normal distribution $\mathcal{N}(0,1)$.

\paragraph{Residual Network (ResNet)} Similarly, we consider a residual network with $L$ hidden layers and a width of $m$. Compared to the fully connected network, in addition to the incorporation of skip connections in the recurrence relation, the skip connections require that the input and output dimensions of each layer be the same, thus necessitating additional transformations at the input layer. Specifically, the details are as follows:
\begin{equation}\label{eq: resnet}
\begin{aligned}
f(\x;\bm{\theta}) &= \bm{W}^{(L+1)} \bm{\alpha}^{(L)}(\x);\quad\bm{\alpha}^{(0)}(\x) = \sqrt{\tfrac{1}{m}} (\bm{A}\x + \bm{b}); \\
\widetilde{\bm{\alpha}}^{(l)}(\x)&=\sqrt{\tfrac{2}{m}}\, \sigma\mpt{\bm{W}^{(l)} \bm{\alpha}^{(l-1)}(\x)+\bm{b}^{(l)}}, \quad l\in[L];\\
\bm{\alpha}^{(l)}(\x) &= \bm{\alpha}^{(l-1)}(\x) + a\sqrt{\tfrac{1}{m}}\mpt{\bm{V}^{(l)}\widetilde{\bm{\alpha}}^{(l)}(\x)+\bm{d}^{(l)}}, \quad l\in[L],
\end{aligned}
\end{equation}
where the parameter $a$ is a chosen scaling factor (normally choose $a=1$ \citep{he2016deep} or $a=L^{-\gamma}$ with $1/2<\gamma\leq 1$ \citep{huang2020deep,belfer2024spectral}). The matrices $\bm{A}$, $\bm{b}$, $\bm{V}^{(l)}$ and $\bm{d}^{(l)}$ have dimensions $\mathbb{R}^{m \times d}$, $\mathbb{R}^{d \times 1}$, $\mathbb{R}^{m \times m}$ and $\mb{R}^{d\times 1}$, respectively. The other parameters and notations are similar to those in the FCN and will not be repeated here. All parameters are nitialized as i.i.d. random variables from $\mathcal{N}(0,1)$  except for $\bm{b}^{(l)}$ and $\bm{d}^{(l)}$, which are initialized at zero (this exceptional setting is merely for the convenience of NTK calculation, which is not the main focus of this paper).

\paragraph{Loss Function} For the classification problem, we apply the commonly used and representative Cross-Entropy loss function. However, our proof can be generalized to a class of loss functions, as mentioned in Remark \ref{rem: general_loss}. That is 
\begin{align*}\begin{gathered}
\mathcal{L}(\bm{\theta}) = -\textstyle\sum\limits_{i=1}^n[y_i \ln (o_i) + (1-y_i)\ln(1-o_i)] = \textstyle\sum\limits_{i=1}^n \ell\big((2y_i - 1)f(\x_i;\bm{\theta})\big),\\
 o_i \coloneqq \frac1{1+\me^{-f(\x_i;\bm{\theta})}}; \qquad \ell(x) \coloneqq \ln(1+\me^{-x}).\end{gathered}
\end{align*}
In this setting, $o_i$ can be regarded as the `output probability', since it transforms $f(\x_i;\bm{\theta})$ and compares it with the label $y_i$. In this way, the output of the network $f(\x;\bm{\theta})$ is expected to be positively correlated with the probability that the label $y$ of $\x$ is $1$ rather than $0$. We use $\bm{u}$ to denote the output residual, whose $i$-th component $u_i$ represents the difference between the output probability $o_i$ and the true label $y_i$, i.e., $u_i \coloneqq |o_i - y_i|$. 
Thus, we have 
\begin{equation}\label{eq: def_u}
u_i = \frac{1}{1 + \me^{(2y_i - 1)f(x_i;\theta)}} = 
\begin{cases}
o_i & \text{if } y_i = 0; \\
1 - o_i & \text{if } y_i = 1.
\end{cases}
\end{equation}
It is easy to check that $u_i \in [0,1]$. 
Under the cross-entropy loss function, we use gradient flow to approximate the gradient descent optimization process. The gradient can be conveniently expressed in terms of output residual $\bm{u}$,  and the dynamic equation is given by:
\begin{align}
\notag\frac{\dd}{\dd t}\bm{\theta}_t&=-\nabla_{\bm{\theta}}\mathcal{L}(\bm{\theta}_t)=-\sum_{i=1}^{n}\ell'\big({(2y_i-1)f(\bm{x}_i;\bm{\theta}_t)\big)}(2y_i-1)\nabla_{\bm{\theta}}f(\bm{x}_i;\bm{\theta}_t)\\
&=\sum_{i=1}^n\frac{(2y_i-1)\nabla_{\bm{\theta}}f(\bm{x}_i;\bm{\theta})}{1+\me^{(2y_i-1)f(\bm{x}_i;\bm{\theta}_t)}}=\sum_{i=1}^n\nabla_{\bm{\theta}}f(\bm{x}_i;\bm{\theta})(2y_i-1)u_i.
\label{eq: parameter_gradient_flow}
\end{align}

\section{Neural Tangent Kernel}\label{sec: NTK}

\subsection{Empirical NTK and NTK: A Brief Overview}
We begin by introducing the empirical NTK and NTK, which are widely employed to analyze the training dynamics of neural networks and have been extensively discussed in the literature.

\paragraph{Empirical Neural Tangent Kernel} The empirical Neural Tangent Kernel (empirical NTK), which is also called the neural network kernel in some literature \citep{lai2023generalizationabilitywideneural, li2024eigenvalue, lai2023generalizationabilitywideresidual}, is the inner product of the derivatives of the neural network with respect to all parameters. We denote the empirical NTK by $K_t(\cdot, \cdot): \mathbb{R}^d \times \mathbb{R}^d \to \mathbb{R}$ . Specifically, the expression is given by
\begin{align*}
 K_t(\x,\x') = \ag{\nabla_{\bm{\theta}} f(\x;\bm{\theta}_t),\nabla_{\bm{\theta}} f(\x';\bm{\theta}_t)} = \nabla_{\bm{\theta}} f(\x;\bm{\theta}_t)^\top\nabla_{\bm{\theta}} f(\x';\bm{\theta}_t). 
\end{align*} 
The dynamics of the parameters during the training process are given by \cref{eq: parameter_gradient_flow}. Since $f(\x;\bm{\theta}_t)$ is a function that determined by $\bm{\theta}_t$, its dynamics can easily be derived from \cref{eq: parameter_gradient_flow} as follows:
\begin{equation}\label{eq: f_gradient_flow}
\begin{aligned}
 \frac{\dd}{\dd t}{f}(\x;\bm{\theta}_t) &
 = \sum_{i=1}^n \nabla_{\bm{\theta}} f(\x;\bm{\theta}_t)^\top \nabla_{\bm{\theta}} f(\x_i ;\bm{\theta}_t) (2y_i -1)u_i = \sum_{i=1}^n K_t(\x,\x_i) (2y_i-1)u_i.
\end{aligned}
\end{equation}
Note that \cref{eq: f_gradient_flow} has a form similar to kernel logistic regression. Unlike traditional kernel regression, the `kernel function' in \cref{eq: f_gradient_flow} is not fixed during training; instead, it changes over time.

\paragraph{Neural Tangent Kernel}
To better understand the dynamic characteristics of the network output function, numerous previous studies \citep{jacot2018neural, allen2019convergence, arora2019exact} have conducted systematic investigations into the properties of the empirical NTK. These studies have demonstrated that in the context of regression problems, as the network width $m$ approaches infinity, the empirical NTK converges to a fixed kernel, both at initialization and during the training process. This fixed kernel is referred to as the Neural Tangent Kernel (NTK) and denoted as $K_{\mr{NT}}$:
\begin{align*}
K_t(\x,\x') \stackrel{\mb{P}}{\longrightarrow} K_{\mr{NT}}(\x,\x'). 
\end{align*}
Moreover, \citet{lai2023generalizationabilitywideneural,li2024eigenvalue,lai2023generalizationabilitywideresidual} further  demonstrate that the convergence of the empirical NTK to the NTK is uniform across all possible input vectors and all time points. Specifically, under certain conditions, it holds with high probability that
\begin{align}\label{eq: convergence}
\sup_{\x,\x'}\sup_{t\geq 0}|K_t(\x,\x')-K_{\mr{NT}}(\x,\x')|=o_m(1),\quad~\text{for}~\lim_{m\to \infty}o_m(1)=0.
\end{align}
This phenomenon significantly contributes to the understanding of the generalization ability of networks within the framework of NTK theory.

In regression problems, the convergence of the empirical NTK relies on the convexity of the MSE loss function. If similar convergence result holds for classification problems, then, based on \cref{eq: f_gradient_flow}, it can be inferred that the dynamic properties of the classification neural network would resemble those of kernel logistic regression. However, in classification problems where cross-entropy is the loss function, we show that this convergence no longer holds. We will elaborate on this point in the following sections, which also highlight the limitations of NTK theory.

\subsection{Strictly Positive Definiteness of the NTK}\label{sec: SPD_of_NTK}
As noted in \citet{caponnetto2007optimal,steinwart2008support,lin2020optimal}, studying the spectral properties of kernels is essential in classical kernel regression. In regression problems, the uniform convergence of the empirical NTK relies on the strictly positive definiteness of the NTK. Therefore, in this subsection, we revisit the key spectral properties of the NTK. 
We will soon see that both the statement and proof of the main conclusion of this paper are also contingent upon positive definiteness.


Since the importance of positive definiteness in this paper, we first explicitly recall the following definition of strictly positive definiteness to avoid potential confusion. 
\begin{definition}[Strictly positive definiteness]
A kernel function $K(\cdot,\cdot): \mc{X} \times \mc{X} \to \mb{R} $ is said to be strictly positive definite (positive semi-definite) over a domain $\mathcal{X}$, if for any positive integer $n$ and any set of $n$ different points $\x_{1},\dots,\x_{n}\in \mathcal{X}$, the smallest eigenvalue $\lambda_{\min}$ of the empirical NTK matrix $K({\X},{\X})=(K({\x}_{i},{\x}_{j}))_{1\leq i,j\leq n}$ is positive (non-negative). 
\end{definition}

To derive the strictly positive definiteness of NTK of fully connected networks and residual networks, we first demonstrate the their expression in the recursive form or explicit form, respectively. For convenience, in this paper, we will treat the following expression as the definition of the NTK for the corresponding neural network. 

\paragraph{NTK of Fully Connected Network (FCNTK)} 
We first present the recursive formula for the NTK of the fully connected network (FCNTK) given by \cref{eq: fc_network}, 
the NTK denoted by $\FCNTK$ can be defined as follows, as shown in \cite{jacot2018neural}. We define
\begin{align*}
\begin{gathered}
 \Theta^{(1)}(\x,\x') = \Sigma^{(1)}(\x,\x') = \langle \x,\x' \rangle+1; \\
 \Theta^{(l+1)}(\x,\x') = \Theta^{(l)}(\x,\x')\, \dot{\Sigma}^{(l+1)}(\x,\x') + \Sigma^{(l+1)}(\x,\x')
 \end{gathered}
 \end{align*}
for $l\in[L]$, where $\Sigma^{(l+1)}$ and $\dot\Sigma^{(l+1)}$ are defined by
\begin{equation*}
\begin{aligned}
 \Sigma^{(l+1)}(\x,\x') &= 2\mpt{\mathbb{E}_{f \sim \mc{N}(0,\Sigma^{(l)})} \mbk{\sigma\big(f(\x)\big)\sigma\big(f(\x')\big)} + 1};\\
 \dot{\Sigma}^{(l+1)}(\x,\x') &= 2\mpt{\mathbb{E}_{f \sim \mc{N}(0,\Sigma^{(l)})} \mbk{\dot{\sigma}\big(f(\x)\big)\dot{\sigma}\big(f(\x')\big)}},
 \end{aligned}
\end{equation*}
and $\dot{\sigma}$ represents the derivative of the ReLU function. Then the formula of $\FCNTK$ is 
\begin{equation*}
 \FCNTK(\x,\x') = \Theta^{(L+1)}(\x,\x').
\end{equation*}



\paragraph{NTK of Residual Network (ResNTK)}
The expression for the NTK of the residual network (ResNTK) is given by \citet{huang2020deep}, but in that work, the parameters of the input and output layers are fixed (i.e., $\bm{A}$, $\bm{b}$, and $\bm{W}^{(L+1)}$ are fixed after random initialization and do not participate in gradient descent training), and there is no bias term in the input layer (equivalent to removing the $\bm{b}$ term in the input layer as in this paper). Considering the impact of these subtle differences and applying the conclusions and methods provided by \citet{huang2020deep}, we can easily derive the NTK expression for the ResNet defined by \cref{eq: resnet} in this paper.
Introduce the following functions: 
\begin{align*}
\kappa_0(u)=\frac{1}{\pi}\mpt{\pi-\arccos u}, \quad
\kappa_1(u)=\frac1\pi\mpt{u\mpt{\pi-\arccos u}+\sqrt{1-u^2}},
\end{align*}
Let $\x,\x'\in\mc{X}$ be two samples, the NTK of an $L$-hidden-layer ResNet, denoted as $\RESNTK(\x,\x')$, is given by
\begin{align*}
\begin{split}
\RESNTK(\bm{x},\bm{x}') &= \big\|\tbinom{\x}{1}\big\|\big\|\tbinom{\x'}{1}\big\|\cdot\bk{K_L(\tilde{\x},\tilde{\x}')+\tilde{\x}^\top\tilde{\x}'\,B_1(\tilde{\x},\tilde{\x}') +a^2\,r(\tilde{\x},\tilde{\x}')};\\
r(\tilde{\x},\tilde{\x}') &= \sum_{l=1}^L B_{l+1} \Big[(1+a^2)^{l-1}\kappa_1\mpt{\tfrac{K_{l-1}}{(1+a^2)^{l-1}}}+(K_{l-1}+1)\cdot\kappa_0\mpt{\tfrac{K_{l-1}}{(1+a^2)^{l-1}}}+1\Big],
\end{split}
\end{align*}
where $\tilde{\x}=\binom{\x}{1}/\big\|\binom{\x}{1}\big\|$, $\tilde{\x}'=\binom{\x'}{1}/\big\|\binom{\x'}{1}\big\|$ and $K_{l}$, $B_l$ are defined by following recursive relation:
\begin{align*}
\begin{gathered}
K_0(\tilde{\bm{x}},\tilde{\bm{x}}')=\tilde{\x}^\top \tilde{\x},\quad K_l(\tilde{\bm{x}},\tilde{\bm{x}}')=K_{l-1}(\tilde{\bm{x}},\tilde{\bm{x}}')+ a^2 (1+a^2)^{l-1}\kappa_1\mpt{\tfrac{K_{l-1}(\tilde{\bm{x}},\tilde{\bm{x}}')}{(1+a^2)^{l-1}}};\\
B_{L+1}(\tilde{\bm{x}},\tilde{\bm{x}}')=1,\quad B_l(\tilde{\bm{x}},\tilde{\bm{x}}')=B_{l+1}(\tilde{\bm{x}},\tilde{\bm{x}}')\left[1+a^2\kappa_0\mpt{\tfrac{K_{l-1}(\tilde{\bm{x}},\tilde{\bm{x}}')}{(1+a^2)^{l-1}}}\right]
\end{gathered}
\end{align*}
for $l\in[L]$. In the above equations, $K_l$ and $B_l$ are abbreviations for $K_l({\tilde{\bm{x}}},\tilde{\bm{x}}')$ and $B_l(\tilde{\bm{x}},\tilde{\bm{x}}')$, respectively.


\paragraph{Strictly Positive Definiteness of the NTK}
For the NTK defined above, we have the following proposition on their strictly positive definiteness, which will be proved in Appendix \ref{sec: SPD}: 
\begin{proposition}[Strictly positive definiteness of NTK]\label{thm: positive_definite}
The NTK of fully connected network \cref{eq: fc_network} and residual network \cref{eq: resnet} is strictly positive definite on a compact set $\mathcal{X}$. 
\end{proposition}

The strictly positive definiteness of NTKs has long been a topic of interest in NTK theory \citep{jacot2018neural, zhang2024unified, nguyen2021tight}. The strictly positive definiteness of the NTK for fully connected networks defined on the unit sphere $\bbS^{d}$ was first proved by \cite{jacot2018neural}. Recently, \cite{lai2023generalizationabilitywideneural} proved the strictly positive definiteness of the NTK for one-hidden-layer biased FCNs on $\R$, and \cite{li2024eigenvalue} generalized this result to multi-layer FCNTKs on $\R^d$. In these works, the bias term $\bm{b}^{(l)}$ is often omitted to simplify the network structure. In this work, we adopt a commonly used FCN with bias terms and use a recursive form to prove strictly positive definiteness. For the NTK of residual networks, we use $\phi(\x)=\binom{\x}{1}/\big\|\binom{\x}{1}\big\|$ to transform the input vector from $\mathbb{R}^d$ to the upper unit hemisphere $\mathbb{S}_+^d:=\mcl{(x_1,x_2,\dots,x_d,x_{d+1})^\top\in\mb{S}^d|x_{d+1}>0}$, and then prove the strictly positive definiteness of the dot-product kernel.


%% file: sec_results.tex
\section{Main Results}\label{sec: Main results}

\subsection{Divergence of the Network}\label{subsec: Divergence of network}
At standard network initialization, the empirical NTK converges in probability as the width tends to infinity \citep{arora2019exact}, a property that is independent of the loss function and training method. 
In regression problems, the empirical NTK is also proven to converge during training, which forms the theoretical basis for studying the generalization ability of neural networks within NTK theory \citep{suh2021non,li2024eigenvalue,lai2023generalizationabilitywideneural}. This phenomenon occurs because the parameters do not deviate significantly from their initial values, regardless of the training duration, which is known as the so-called lazy regime or NTK regime \citep{allen2019convergence}. However, this no longer holds when training with the cross-entropy loss function.
Let denote $\widetilde{\lambda}_0(t)$ denote the minimum eigenvalue of the empirical NTK matrix, i.e., $\widetilde{\lambda}_0(t) = \lambda_{\min}(K_t(\X,\X))$. Then, we have
\begin{theorem}\label{thm: divergence_of_network}
Fix the training samples $\{ (\x_i,y_i) \}_{i \in [n]}$.Consider fully-connected networks and residual networks with cross-entropy loss function in classification problems. If $\widetilde{\lambda}_0(t)$ is consistently lower bounded by some postive constant $C$ during training, then the network output function will tend to infinity at the sample points $\{\x_i\}_{i\in [n]}$, i.e., $\lim\limits_{t \to \infty}|f_t(\x_i)| = \infty$.
\end{theorem}
The above result is expected: in order for the loss function to approach zero, the function values at the sample points must inevitably tend to infinity. However, this does not mean the result is trivial, because in classification problems, the loss function does not always behave this way, such as in logistic regression with linearly inseparable data. For kernel logistic regression, if its kernel function is strictly positive definite, then such exceptional cases can be avoided. The eigenvalue condition of the empirical NTK matrix in the above theorem plays a similar role.

Since the output function value of the network tends to infinity, this means that some parameters of the network will diverge during the training process as time increases. Consequently, we can demonstrate the following direct corollary of Theorem \ref{thm: divergence_of_network}:
\begin{corollary}[Failure of the NTK regime]
Assume the conditions in Theorem \ref{thm: divergence_of_network} holds.For any initialized parameter $\bm{\theta}_0$, after training we have that $\limsup\limits_{t \to \infty}\norm{ \bm{\theta}_t - \bm{\theta}_0}_{\infty} = \infty$.
\end{corollary}

Note that in Theorem \ref{thm: divergence_of_network},we do not impose any requirements on the width of network, meaning that the aforementioned divergence holds regardless of the width $m$.
Of course, we have not forgotten that all of this is based on the condition $\widetilde{\lambda}_0(t) \geq C > 0$. This condition will be satisfied if empirical NTK uniformly converges to NTK, which will be used in the proof by contradiction of Theorem \ref{thm: divergence_of_ntk}.

\subsection{Divergence of the Empirical NTK}
In this subsection, we will discuss the divergence of empirical NTK in classification problem and demonstrate our main results. In infinite time, the empirical NTK continues to evolve and will not converges to a fixed NTK like in the regression problem case. Now we demonstrate our main theorem.

\begin{theorem}[Divergence of the empirical NTK]\label{thm: divergence_of_ntk}For fully connected network defined by $\cref{eq: fc_network}$ and residual network defined by \cref{eq: resnet}, given samples $\{(\x_i,y_i)\}_{i=1}^n$ and consider the training process under cross-entropy function. 
Let $\lambda_0 \coloneqq \lambda_{\min}(K_{\mr{NT}}(\X,\X))$. For any initialized parameter $\bm{\theta}_0$, there exists $\x,\x' \in \mc{X}$, such that 
\begin{align*}
\sup_{t \geq 0} \lvert K_t(\x,\x') - K_{\mr{NT}}(\x,\x')\rvert \geq \frac{ \lambda_0}{2n^2}. 
\end{align*}
\end{theorem}
To understand the significance of this result, we compare it with regression problems: 
In regression problems, the distance between the empirical NTK and the NTK is bounded by an infinitesimally small upper bound (with respect to the network width $m$) throughout the entire training process (see \cref{eq: convergence}). This allows the training of neural networks to be well approximated by NTK regression when network widths $m$ is large enough. However, in classification problems, such an infinitesimal upper bound no longer holds. Instead, there exists a positive lower bound that holds for all network widths $m$: although the empirical NTK and NTK remain close in the early stages of training for wide networks \citep{ji2019polylogarithmic, taheri2024generalization}, 
they eventually diverge as training progresses. 
\newpage
This result suggests that the NTK theory, which works well for regression problems, does not directly apply to classification problems, thereby necessitating new approaches to understand and analyze the behavior of classification neural networks


%% file: sec_proof_sketch.tex
\section{Proof Sketch}\label{Proof Sketch}

We present the sketch of our proofs in this section and defer the complete proof to Appendix. In this section, we denote the network output function by $f_t^{\mr{NN}} :\mc{X} \subset \mb{R}^d \to \mb{R}$, with the corresponding empirical NTK and NTK represented by $K_t^m(\cdot,\cdot)$ and $K_{\mr{NT}}(\cdot,\cdot)$, respectively. Let $\{(\x_i,y_i)\}_{i\in[n]}$ be the fixed set of training samples. 
For simplicity, in this proof sketch, we use the same symbols for different network structures, meaning these symbols refer to both fully connected and residual neural networks. We assume the neural network is trained using the cross-entropy loss function on the training samples.

We can complete the proof of Theorem \ref{thm: divergence_of_ntk} by contradiction. For the sake of contradiction, we first assume that during training, the empirical NTK remains approximately invariant:
\begin{equation*}
 \sup_{\x,\x' \in \mc{X}} \sup_{t \geq 0} \lvert K_t^m(\x,\x') - K_{\mr{NT}}(\x,\x') \rvert = o_m(1).
\end{equation*}
This implies that we also have the following relationship between the minimum eigenvalues:
\begin{equation*}
 \sup_{t \geq 0 } \abs{\widetilde{\lambda}_{0}(t) - \lambda_{\min} \big(K_{\mr{NT}}(\X,\X) \big) } \leq \sup_{t \geq 0} \norm{K_t^m(\X,\X) - K_{\mr{NT}}(\X,\X)}_{\mr{F}} = o_m(1),
\end{equation*}
where $\widetilde{\lambda}_{0}(t)=\lambda_{\min}(K^m_t(\X,\X))$. 
Recall that the strictly positiveness of the NTK has been established in \cref{sec: SPD_of_NTK}, so we know that the empirical NTK is also consistently strictly positive definite during training. Specifically, there exists a positive constant $C$ such that $\widetilde{\lambda}_0(t)\geq C>0$.
Since the empirical NTK involves the derivatives of the network with respect to the parameters, it can be used to analyze the dynamics of the network function during training. As defined in \cref{eq: def_u}, we let $\bm{u} \in \mathbb{R}^n$ represent the output residual, which also reflects the output values of the network function on the training set. We construct the following Lyapunov function
\begin{equation*}
 V(\bm{u}) \coloneqq \textstyle\sum\limits_{i=1}^n \ln (1-u_i(t)), \qand V_t \coloneqq V(\bm{u}(t)).
\end{equation*}
For the scalar $V_t$, which satisfie $V_t\geq 0$, the dynamical equation is given by
\begin{equation*}
 \frac{\dd}{\dd t} V_t = - \bm{u}(t)^\top \big[{\diag(2\bm{y} - 1)^\top K_t^m(\X,\X) \diag (2\bm{y}-1)}\big] \bm{u}(t).
\end{equation*}
Since the empirical NTK is strictly positive definite, with $\widetilde{\lambda}_0(t)\geq C >0$, we know that $V_t$ is monotonically decreasing. Through some technical analysis (details in the appendix), we finally get $V_t \to 0$ as $t \to \infty$. Therefore, as $t \to \infty$, we have $\bm{u}(t)\to \bm{0}$ and $\lvert f_t^{\mr{NN}}(\x_i) \rvert \to \infty$.
Combined with the specific network structure, we can derive that
\begin{equation*}
 \lim_{t \to \infty} \sup_{\x,\x' \in \{\x_i\}_{i=1}^n } \lvert K_t^m(\x,\x') \rvert \to \infty,
\end{equation*}
which indicates that the empirical NTK diverges during training. Since the assumption of uniform convergence of the empirical NTK to the NTK implies that the empirical NTK is bounded, this leads to a contradiction, completing the proof.

%% file: sec_experiments.tex
\section{Numerical Experiments}\label{sec: Numerical Experiments}
In this section, we present several numerical experiments to illustrate the theoretical results in this paper. For convenience, we only show the experimental results for the fully connected network, as similar results can be obtained for residual networks.

\subsection{Synthetic Data}
We conduct experiments on a synthetic dataset using the fully connected network described earlier, which consists of three hidden layers. For visualization purposes, we restrict the input to the unit circle $\mathbb{S}^1$, where the input is represented as $(\cos\theta, \sin\theta)^\top$, with $\theta$ being the angle between the input vector and the $x$-axis, referred to as the polar angle. Therefore, the input dimension of the network is $d=2$. Our experiments are organized as follows.


\begin{wrapfigure}{r}{0.58\linewidth}
 \centering
 \includegraphics[width=0.52\textwidth]{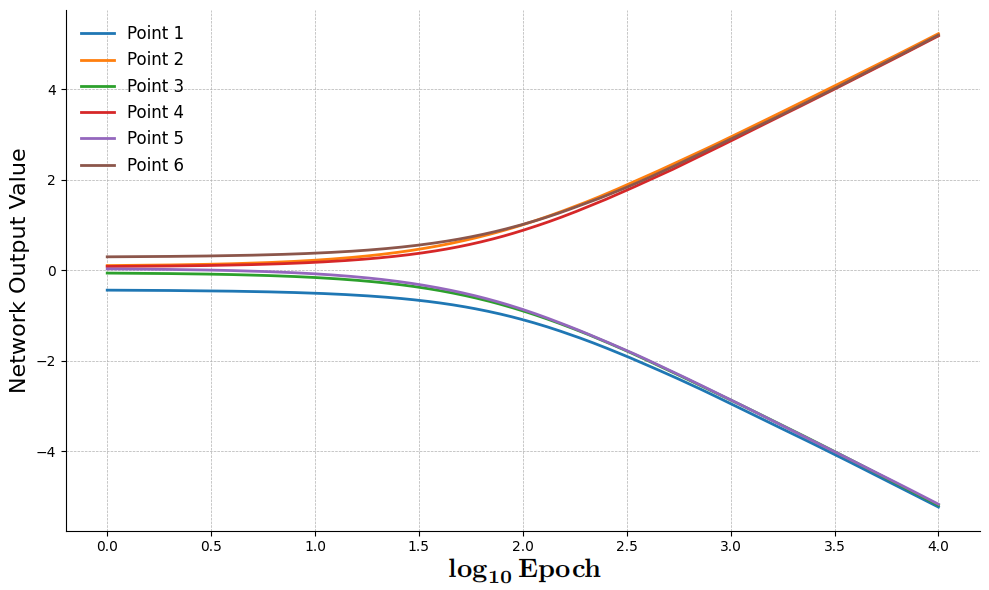} 
 \vspace{-3mm}
 \caption{Divergence of the Fully Connected Network During Training: The plot shows the output values of the network at the six training points throughout the training process. Despite the alternating labels, the network function diverges at all six points as training progresses.}
 \vspace{-3mm}
 \label{fig: divergence_of_fc_network}
\end{wrapfigure}

\paragraph{Divergence of the Fully Connected Network Function During Training} We set the width of the fully connected network to $m = 2000$ and construct a training set consisting of six input vectors uniformly distributed on the unit circle $\mathbb{S}^1$. Specifically, the points $\{\x_i\}_{i \in [6]}$ are given by $(\cos \theta_i, \sin \theta_i)^\top$ with $\theta_i = {i \pi}/{3}$. The corresponding labels are $(0,1,0,1,0,1)$. The alternating pattern of the labels is designed to mitigate any potential effects of data separability. We train the network for 10,000 epochs with a learning rate of $0.1$. We plot the network output values at the six training points during training, and the final result is shown in Figure \ref{fig: divergence_of_fc_network}.


\paragraph{Convergence of the Empirical NTK at Initialization}
We plot the empirical NTK function of the network at initialization for different widths by fixing the first input $\x = (1,0)^\top$ and varying the polar angle $\theta$ of the other input $\x'$ from $-\pi$ to $\pi$. The result is shown in Figure \ref{fig: convergence_of_fc_ntk}, indicating that the empirical NTK converges at $t=0$. The results are consistent with previous findings in NTK theory. 
\begin{figure}[H]
 \centering
 \includegraphics[width=\linewidth]{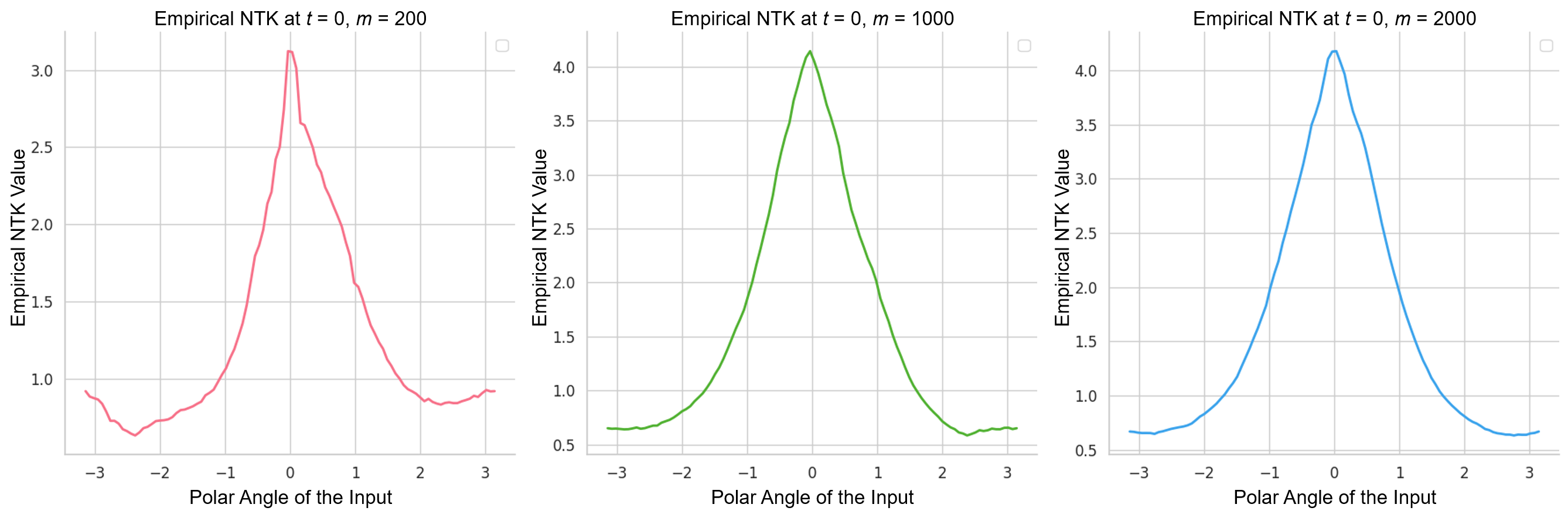}
 \vspace{-5mm}
 \caption{Convergence of the Fully Connected Empirical NTK at Initialization: The empirical NTK functions at initialization for networks with different widths ($m=200$, $m=1000$, and $m=2000$) are shown. The image shows that as $m$ increases, the empirical NTK values become more stable. The results indicate that the empirical NTK converges at $t=0$. }
 \label{fig: convergence_of_fc_ntk}
\end{figure}

\begin{wrapfigure}{r}{0.58\linewidth}
 \centering
 \includegraphics[width=0.48\textwidth]{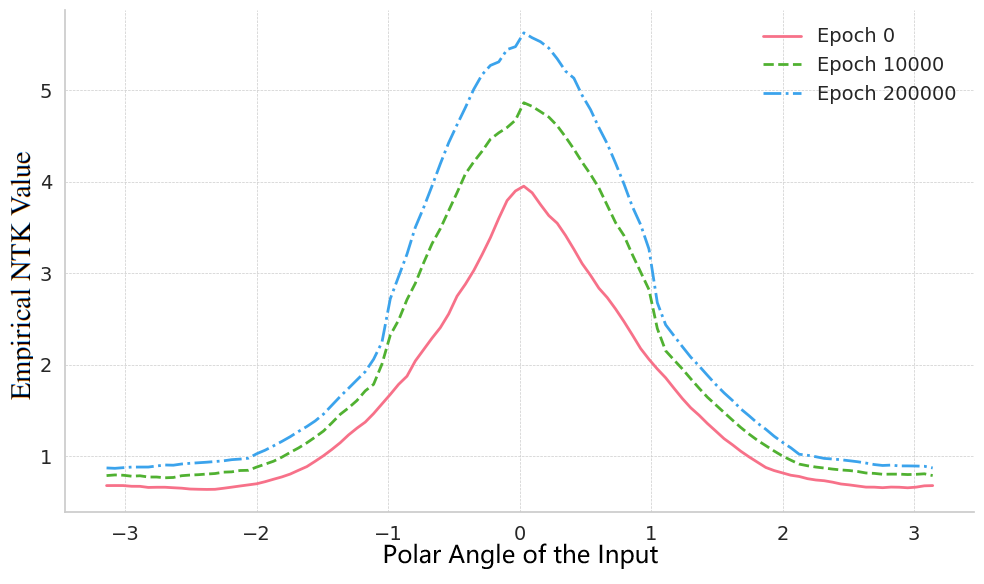} 
 \vspace{-3mm}
 \caption{Divergence of the Fully Connected Empirical NTK During Training: The behavior of the empirical NTK at different epochs during training for a network with width $m=2000$. The plot highlights the divergence of empirical NTK over the course of training.}
 \vspace{-5mm}
 \label{fig: divergence_of_fc_ntk}
\end{wrapfigure}

\paragraph{Divergence of the Empirical NTK During Training}
Next, we examine the behavior of the empirical NTK during training to highlight the impact of the cross-entropy loss function. We train the network with a width of $m = 2000$ using the cross-entropy loss function. As in the previous experiment, we fix the first input $\x = (1,0)^\top$ and vary the polar angle $\theta$ of the other input $\x'$ from $-\pi$ to $\pi$. We then plot the network output at these points. Figure \ref{fig: divergence_of_fc_ntk} presents the evolution of the empirical NTK across different epochs, highlighting its divergence during the training process. This behavior stands in stark contrast to what observed in regression problems.



\subsection{Real Data}
We conduct an experiment on the MNIST dataset, using parity (odd or even) as the criterion for binary classification. We train a four-layer fully connected neural network for this task. The network has a width of $m = 500$, with a learning rate of $\text{lr} = 0.5$, and is trained for $\text{epoch} = 100,000$. Since the dimension of the MNIST dataset is $d = 784$, it is difficult to visualize the empirical NTK values as shown in Figures \ref{fig: convergence_of_fc_ntk} and \ref{fig: divergence_of_fc_ntk}. Therefore, we select three points and show that the empirical NTK values diverge at these points, which is sufficient to confirm our results. The results are presented in Figure \ref{fig: The NTK value on MNIST dataset}. The values from the first 10 epochs are discarded to better observe the training behavior and minimize the impact of initialization on the experimental results.

\begin{figure}[H]
 \centering
 \includegraphics[width=\linewidth]{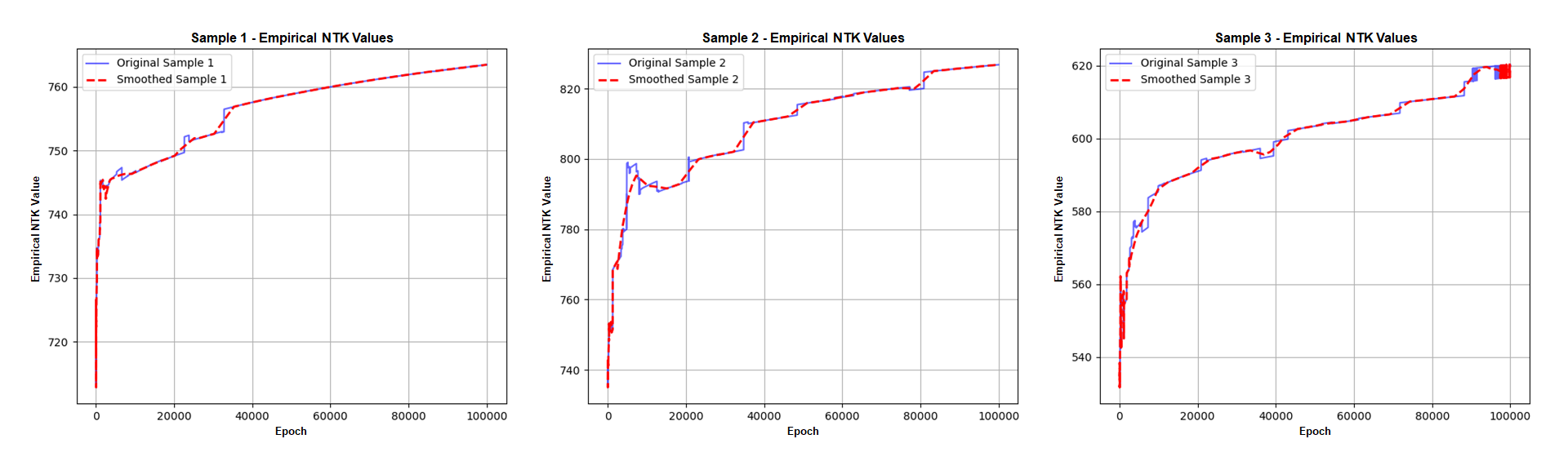}
 \vspace{-7mm}
 \caption{The Empirical NTK Values on the MNIST Dataset for Three Selected Points: The blue lines represent the original empirical NTK values, while the red dashed lines show the smoothed values. The empirical NTK values are computed over $100,000$ epochs for a four-layer fully connected neural network with a width of $m = 500$ and a learning rate of $\text{lr} = 0.5$. These points are selected to illustrate the divergence of the empirical NTK values.}
 \label{fig: The NTK value on MNIST dataset}
\end{figure}


%% file: sec_discussion.tex
\section{Discussion}\label{sec: Discussion}
\paragraph{Conclusion}
In this paper, we study the behavior of the empirical NTK for fully connected networks and residual networks in classification problems, with a particular focus on their convergence properties when trained using the cross-entropy loss function. Our study reveals that, unlike regression problems where the convergence of the empirical NTK is well documented, the empirical NTK in classification tasks does not uniformly converge to the NTK aross all time. 
This property indicates a limitation of the NTK theory when applied to classification problems and suggests that the NTK does not serve as a fixed approximation of the training dynamics in these scenarios.

Our analysis highlights the need for new theoretical tools and frameworks to better understand and analyze neural networks in the context of classification. We have shown that the standard NTK regime fails in classification problems, which implies that current theories might not fully explain the generalization properties of networks trained on such tasks. This opens up new avenues for research in understanding the complex dynamics of neural networks beyond the NTK framework, especially when dealing with classification problems where cross-entropy loss is prevalent.

\paragraph{Limitation and Future Work}
This paper proves that the empirical NTK does not converge uniformly to the NTK, which is sufficient to compare it with regression problems and highlight the limitations of the NTK theory. However, this qualitative result is relatively preliminary and lacks precision. Based on the experimental results in \cref{sec: Numerical Experiments}, the empirical NTK appears to tend to infinity, but we have not drawn a definitive conclusion (either affirmative or negative; note that the two conclusions in \cref{subsec: Divergence of network} are only part of a proof by contradiction, and their premises have not been positively proven), nor have we provided more refined quantitative results, such as divergence rates. These aspects remain open questions and warrant further exploration in future research.

On the other hand, this paper does not provide a better solution to the limitations of NTK theory. In fact, in recent years, some researchers have pointed out the limitations of NTK theory (including for regression problems) from different perspectives and have made preliminary explorations \citep{li2025diagonaloverparameterizationreproducingkernel}, but progress remains limited. These explorations are challenging yet highly meaningful, and guiding such efforts is one of the goals of this paper.


%% file: sec_proof.tex
\newpage

\section{Further Notations}

We first demonstrate the definition of Gaussian process. 
\begin{definition}
A stochastic process $f(\x)$, $\x \in \mathbb{R}^d$, is called a Gaussian process if for any finite set of points $\x_1, \x_2, \dots, \x_n \in \mathbb{R}^d$, the random vector $(f(\x_1), f(\x_2), \dots, f(\x_n))^\top$ follows a multivariate normal distribution. Specifically, a Gaussian process $f$ is completely specified by its mean function $m(x) = \mathbb{E}[f(x)]$ and covariance function $K(x, x') = \text{Cov}(f(x), f(x'))$, and is denoted as:
$$
f \sim \mathcal{GP}(m(x), K(x, x')).
$$
If $m(x) = 0$, the process is called a centred Gaussian process.
\end{definition}

For the notational simplicity, we denote a centred Gaussian process $f \sim \mathcal{GP}(0, K)$ directly by $f \sim K $.

\section{The strictly positive definiteness of NTK}\label{sec: SPD}

In this section, we demonstrate the proof of Proposition \ref{thm: positive_definite}. In the following two subsections, we respectively prove the strictly positive definiteness of the NTK for fully connected network and residual network. For the sake of writing and referencing convenience, we have divided Proposition \ref{thm: positive_definite} into Proposition \ref{prop: positive_definite_ntk_fc} and Proposition \ref{prop: positive_definite_ntk_res}.

\subsection{Fully connected network}


The proof of the strictly positive definiteness of NTK is roughly similar to that in \cite{jacot2018neural}. However, since our network structure includes several bias terms, which is different from \cite{jacot2018neural}, we provide the whole proof here to ensure completeness.

\begin{proposition}\label{prop: positive_definite_ntk_fc}
 The NTK of fully connected network \ref{eq: fc_network} is strictly positive definite on $\mc{X}$. 
\end{proposition}

\begin{proof}[Proof of Proposition \ref{prop: positive_definite_ntk_fc}]
 Recall that the Hadamard product of positive definite matrices is still positive definite, and the sum of positive definite matrices is also still positive definite. 
 By the recursive form of NTK, the result following from Lemma \ref{lem: positive definiteness of ntk}. 
\end{proof}
The following lemma is provided to ensure the recursive proof between layers can proceed.
\begin{lemma}
 If kernel function $K_1 : \mb{R}^d \times \mb{R}^d \to \mb{R}$ is strictly positive definite, then the kernel function $K_2$ defined by $K_2(x,x') \coloneqq \mathbf{E}_{f\sim K_1}[\sigma(f(x))\sigma(f(x'))]$ is also strictly positive definite. 
\end{lemma}

\begin{proof}
 We prove the lemma by contradiction. Suppose that there exists a set of $ \{x_i\}_{i =1 ,\cdots ,n} \subset \mb{R}^d $ such that the Gram kernel matrix satisfies that $ K_2(X,X)$ is not strictly positive definite. Firstly, it is direct to verify that $K_2(X,X)$ is positive definite. For any $u \in \mb{R}^n$, we have 
 \begin{equation}
 u^\top K_2(X,X) u = \mathbf{E}_{f \sim K_1}[ (u^\top \sigma(f(X)))^2 ] \geq 0.
 \end{equation}
 Since we suppose that $K_2(X,X)$ is not strictly positive definite, it means that there exists a vector $u_0 \in \mb{R}^n$ and $ u_0 \neq 0$ such that $ u_0^\top K_2(X,X)u_0 = 0$. Thus we have 
 \begin{equation}
 u_0^\top K_2(X,X) u_0 = \mathbf{E}_{f \sim K_1}[ (u_0^\top \sigma(f(X)))^2 ], 
 \end{equation}
 namely, $u_0^\top\sigma(f(X))=0 $ holds almost surely. However, $f(X)$ is a $n$-dimensional Gaussian variable with strictly positive definite Covariance matrix, and thus is non-degenerate. We define $g(v) = \sum_{i=1}^n u_{0,i} \sigma(g(v_i)$, where $u_{0,i}$ is the $i$-th entry of $u_0$. Without loss of generality, we assume that $ u_{0,n} \neq 0 $. 
 If $u_0^\top\sigma(f(X))=0 $ holds almost surely, we have $g(v) = 0$ almost everywhere with respect to Lebesgue measure. 
 We fix a vector $w \in \mb{R}^{n-1}$, and define the new function by $g_n(v_n) = g( (w^\top,v_n)^\top ) = \sum_{i=1}^{n-1} u_{0,i} g(w_{i}) + u_{0,n} g(v_{n}).$ Since $\sigma$ is non-constant, we have $m_1(A_{w})>0$ where $m_k$ denotes the $k$-dimensional Lebesgue measure and $A = \{v_n | g_n(v_n) \neq 0 \} $. Then $m_n( \{ v | g(v) > 0 \} ) = \int_{\mb{R}} \int_{\mb{R}^{n-1}} \boldsymbol{1}(v_n \in A_w) \dd w \dd v_n >0 $ 
 Therefore, we prove the lemma by contradiction. 
\end{proof}





\begin{lemma}\label{lem: SPD S_+}
 Let $f: [-1,1] \longrightarrow \mb{R}$ be a continuous function with the expansion 
 \begin{equation}
 f(u) = \sum_{n=0}^\infty a_n u^n, \quad u \in [-1,1],
 \end{equation}
 and $k(x,y) = f(\langle x,y \rangle)$ be the dot-product kernel on $\mb{S}^d$. Then if $a_n \geq 0$ for all $n \geq 0$ and there are infinitely many $a_n > 0$, then $k$ is strictly positive definite on $\mathbb{S}_+^d \coloneqq \{ (x_1 ,x_2, \cdots, x_d) \in \mb{S}^d | x_d > 0 \}$. 
\end{lemma}

\begin{proof}
 Let $x_1,\cdots , x_n \in \mb{S}^d_+$ be different. The kernel matrix is 
 \begin{equation}
 k(X,X) = \sum_{n=0}^\infty a_n M_n, \quad M_n=\left( \langle x_i, x_j \rangle^n \right)_{i,j}.
 \end{equation}
 Since $|\langle x_i , x_j \rangle| < 1$, we have
 \begin{equation}
 M_n \longrightarrow I_n,
 \end{equation}
 which is strictly positive definite. In this way, we finish the proof. 
\end{proof}

\begin{lemma}\label{lemma, M coefficients k0 k1}
The coefficients of Maclaurin expansion
 of $\kappa_0(u),\kappa_1(u)$ are both non-negative and infinitely many terms are positive.
\end{lemma}
\begin{proof}
A direct calculation leads to that 
\begin{align}\label{kappa0Taylor}
\kappa_0(u)=\frac12+\frac1\pi\sum_{n=0}^\infty \frac{(2n)!}{4^n(n!)^2(2n+1)}u^{2n+1},
\end{align}
and 
\begin{align}
\kappa_1(u)&=\frac1\pi\left[u\left(\frac\pi2+\sum_{n=0}^\infty \frac{(2n)!}{4^n(n!)^2(2n+1)}u^{2n+1}\right)+1+\sum_{n=1}^\infty\frac{(-1)^{n-1}(2n)!}{4^n(n!)^2(2n-1)}(-1)^nu^{2n}\right]\nonumber\\
&=\frac1\pi+\frac u 2+\frac1{2\pi}\sum_{n=0}^\infty \frac{(2n)!}{4^{n}(n!)(n+1)!}\frac{u^{2n+2}}{2n+1}.\label{kappa1Taylor}
\end{align}
\end{proof}


\begin{lemma}\label{lem: positive definiteness of ntk}
 The kernel function $\Sigma^{(2)}$ is strictly positive definite on $\mc{X}$. 
\end{lemma}

\begin{proof}
 We do a transformation from $\mb{R}^d$ to $\mb{S}_+^{d}$ to finish the proof. For $x\in \mb{R}^d$, we define $\widetilde{x}\coloneqq (x,1) $, which means that add an entry $1$ to the end of the vector $x$. Define $\phi: \mb{R}^d \to \mb{S}_+^{d} $ by $ \phi(x) = \frac{\widetilde{x}}{ \norm{\widetilde{x}}_2}$. 

 If we define a kernel function $K_1(x,x') = \langle x ,x' \rangle $ on $\mb{S}_+^{d} \times \mb{S}_+^{d}$ and
 a kernel function $ K_2(x,x') = \mathbf{E}_{f \sim K_1} [ \sigma(f(x)) \sigma(f(x')) ] = \frac{1}{2} \kappa_1(\langle x,x' \rangle)$, then $K_2$ is strictly positive definite on $\mb{S}^d_+$ by Lemma \ref{lemma, M coefficients k0 k1} and Lemma \ref{lem: SPD S_+}. Actually, we then have that 
 \begin{equation}
 \Sigma^{(2)}(x,x') = 2\norm{\widetilde{x}}_2 K_2(\phi(x),\phi(x')) \norm{\widetilde{x}}_2. 
 \end{equation}
 Therefore, we can directly verify that $\Sigma^{(2)}$ is strictly positive definite on $\mc{X}$. 
\end{proof}

\subsection{Residual network}

For residual neural networks, since only the first layer in the standard residual network structure includes a bias term, we transform the original kernel into an dot-product kernel on the sphere and then take a Maclaurin expansion to complete the proof.

\begin{proposition}\label{prop: positive_definite_ntk_res}
The NTK of residual network \cref{eq: resnet} is strictly positive definite on $\mc{X}$. 
\end{proposition}

 \begin{proof}
 Note that in Lemma \ref{lem: positive definiteness of ntk}, we accomplish the proof through a transformation from $ \mb{R}^d \to \mb{S}_+^d$. 
 We can also define the NTK of ResNet in a similar mannar. It is well known that the sum of a positive definite kernel and a strictly positive definite kernel is still strictly positive definite. Therefore, to derive the strictly positive definiteness of $\RESNTK$, we only need to consider the strictly positive definiteness of $ K_L(\cdot,\cdot)$ on $\mb{S}^{d}_+$, which is shown in Lemma \ref{lem: strictly positive definite}. 

 \end{proof}
 
\begin{lemma}\label{lem: strictly positive definite}~
$K_L(\cdot,\cdot)$ is strictly positive definite on $\S_+^{d}$ when $L\geq 2$.
\end{lemma}

Now we start to prove Lemma \ref{lem: strictly positive definite}. We first prove that
\begin{lemma}\label{lem:strictly positive definite basis}
$K_1(\cdot,\cdot)$ is strictly positive definite on $\S_+^{d}$.
\end{lemma}
\begin{proof}

We have known that $K_1$ is a dot-product kernel. If we denote  by $u = x^Tx'$. Then we have 
\begin{equation}
    K_1(x,x') = u + \alpha^2 \kappa_1(u).
\end{equation}
Thus $K_1(\cdot,\cdot)$ is strictly positive definite on $\mb{S}_+^d$ through Lemma \ref{lem: SPD S_+} and Lemma \ref{lemma, M coefficients k0 k1}.
\end{proof}
Then one can easily prove Lemma \ref{lem: strictly positive definite} recursively based on the result of Lemma \ref{lem:strictly positive definite basis}.

\section{Divergence of the Empirical NTK Under Infinite Time}

\subsection{The Dynamic of Output Residual}

Before the characteristic of the empirical NTK, we first demonstrate the dynamic property of the auxiliary variable defined in \cref{eq: def_u}, which is pivotal in our analysis. 
For $\bm{u} = (u_1,u_2,\cdots,u_n)^\top \in \mb{R}^n$, we have the dynamic equation: 
\begin{align*}
 \frac{\dd}{\dd t} u_i(t) &= \frac{-\me^{(2y_i -1) f_t(\x_i)}}{\left[1+\me^{(2y_i-1)f_t(\x_i)}\right]^2} \cdot (2y_i -1) \frac{\dd}{\dd t} f_t(\x_i) \\&= -u_i(1-u_i)\sum_{j=1}^n(2y_i-1) K_t(\x_i,\x_j)(2y_j-1)u_j.
\end{align*}
If we define the matrix
\begin{align*}
 \bm{K}^r :=&~ \big((2y_i - 1)K_t(\x_i,\x_j)(2y_j-1)\big)_{j\in [n]}^{i \in [n]} 
 \\  =&~ \diag(2\y-1)^\top K_t(\X,\X)\,\diag(2\y-1)
 \\  =&~ \mpt{\bm{k}_1^r, \bm{k}_2^r,\cdots, \bm{k}_n^r}=\mpt{\bm{k}_1^r, \bm{k}_2^r,\cdots, \bm{k}_n^r}^\top,
 \end{align*}
then we can derive a more explicit dynamic equation
\begin{equation}\label{eq: dynamic_of_u}
 \frac{\dd}{\dd t}u_i(t) =-u_i(1-u_i) \bm{k}^{r,\top}_i \bm{u}. 
\end{equation}
It is worth noting that $\bm{K}^r$ is also a symmetric matrix as $K_t(\X,\X)$.

\subsection{The Explicit Formula of the Empirical NTK}

 In this subsection, we introduce the explicit formula of the empirical NTK of the fully connected network.

We denote by $ D_{\x}^{(l)} = \boldsymbol{1}( W^{(l)}\alpha^{(l-1)}(x) + b^{(l)} >0)$, such that we have 
\begin{equation}
 \alpha^{(l)}(x) = \sqrt{\frac{2}{m_l}} D_x^{(l)} (W^{(l)} \alpha^{(l-1)}(x) + b^{(l)} ).
\end{equation}
Then the explicit formula of the empirical NTK of multi-layer fully connected network is
\begin{equation}\label{eq: formula_of_NTK}
\begin{aligned}
 K_t(x,x') &= \sum_{l=1}^{L+1} \left\langle \nabla_{W^{(l)}} f(x) , \nabla_{W^{(l)}} f(x) \right\rangle + \sum_{l=1}^{L} \left\langle \nabla_{b^{(l)}} f(x) , \nabla_{b^{(l)}} f(x) \right\rangle 
 \end{aligned}
\end{equation}
where the explicit formula is
\begin{equation}
\begin{aligned}
 & \sum_{l=1}^{L+1} \left\langle \nabla_{\bm{W}^{(l)}} f(x) , \nabla_{\bm{W}^{(l)}} f(x) \right\rangle \\&= \sum_{l=1}^{L+1} 
 \left\langle \left[ \prod_{r={l+1}}^{L+1} \sqrt{\frac{2}{m_{r-1}}} W^{(r)}D_x^{(r-1)} \right] \alpha^{(l-1)}(x), \left[ \prod_{r={l+1}}^{L+1} \sqrt{\frac{2}{m_{r-1}}} W^{(r)}D_{x'}^{(r-1)} \right] \alpha^{(l-1)}(x') \right\rangle 
 \\& = \sum_{l=1}^{L} \left\langle \left[ \prod_{r={l+1}}^{L+1} \sqrt{\frac{2}{m_{r-1}}} W^{(r)}D_x^{(r-1)} \right] \alpha^{(l-1)}(x), \left[ \prod_{r={l+1}}^{L+1} \sqrt{\frac{2}{m_{r-1}}} W^{(r)}D_{x'}^{(r-1)} \right] \alpha^{(l-1)}(x') \right\rangle
 \\& \quad + \left\langle \alpha^{(L)}(x),\alpha^{(L)}(x') \right\rangle,
 \end{aligned}
\end{equation}
and
\begin{equation}
 \sum_{l=1}^{L} \left\langle \nabla_{b^{(l)}} f(x) , \nabla_{b^{(l)}} f(x) \right\rangle = \sum_{l=1}^L \left\langle \left[ \prod_{r={l+1}}^{L+1} \sqrt{\frac{2}{m_{r-1}}} W^{(r)}D_x^{(r-1)} \right] , \left[ \prod_{r={l+1}}^{L+1} \sqrt{\frac{2}{m_{r-1}}} W^{(r)}D_x^{(r-1)} \right] \right\rangle 
\end{equation}

\subsection{Divergence of the Empirical NTK under Infinite Time}

We first introduce new notations. By Proposition \ref{thm: positive_definite} , we know the NTK is positive definite and thus the kernel matrix $K(X,X)$ is positive definite for any samples $X$. 
Denote $\lambda_0 \coloneqq \lambda_{\min}(K(X,X)) > 0$ to be the minimal eigenvalue of NTK Gram matrix. Also, Let $\widetilde{\lambda}_0(t) \coloneqq \lambda_{\min}(K_t(X,X))$.


\begin{proof}[Proof of Theorem \ref{thm: divergence_of_ntk} \newline]
\textbf{Part 1: Fully Connected Network.}
 We proof the theorem by contradiction. We first assume that there is a kind of parameter initialization such that 
 \begin{equation}\label{eq: unif_conv_assumption_fc}
 \sup_{t\geq 0} \left \lvert K_t(x_i,x_j) - \FCNTK(x_i,x_j) \right \rvert \leq \frac{\lambda_0}{2n^2},
 \end{equation}
 holds for any $x_i ,x_j \in \mc{X}$. Proposition \ref{thm: positive_definite} ensures that $\lambda_0 >0$. Through Lemma \ref{lem: convergence_of_min_lambda} we have 
 \begin{equation}
 \inf_{t\geq 0} \widetilde{\lambda}_0(t) \geq \frac{\lambda_0}{2}>0.
 \end{equation}
 By Lemma \ref{lem: conv_of_function}, we have $(2y_i-1)f_t(x_i) \to +\infty$ as $t \to \infty$. 

 Recall that $f_t(x_i) = W^{(L+1)}_t \alpha^{(L)}_t(x_i) = \sum_{j=1}^{m_L} W^{(L+1)}_{j,t} \alpha^{(L)}_{j,t}(x_i) $, where $\alpha_{j,t}^{(L)}(x_i)$ and $W_{j,t}^{(L+1)}(x_i)$ denotes the $j$-th entry of $ \alpha_{t}^{(L)}(x_i) $ and $ W_{t}^{(L+1)}(x_i)$. 
 Therefore, there exists an index $j_0 \in [m_L]$, such that the limit superior of $W^{(L+1)}_{j_0,t} \alpha^{(L)}_{j_0,t}(x_i ) $ is infinity, as time $t$ comes to infinity. 

 We first consider the case that the limit superior of $\alpha_{j_0,t}^{(L)}(x_i)$ is infinity. 
 In \eqref{eq: formula_of_NTK}, directly we can see, 
 \begin{equation}
 \begin{aligned}
 K_t(x_i,x_i) & \geq \sum_{l=1}^{L+1} \langle \nabla_{W^{(l)}} f(x) , \nabla_{W^{(l)}} f(x) \rangle \geq \langle \alpha^{(L)}(x), \alpha^{(L)}(x) \rangle
 \\& = \sum_{j=1}^{m_L} \alpha^{(L)}_{j,t}(x_i) \alpha^{(L)}_{j,t}(x_i) \geq \alpha^{(L)}_{j_0,t}(x_i) \alpha^{(L)}_{j_0,t}(x_i) \to \infty.
 \end{aligned}
 \end{equation}
 It yields the result that $K_t(x_i, x_i)$ diverges to infinity, which contradicts to the assumption in \eqref{eq: unif_conv_assumption_fc}. 

 We then consider the other case that the limit superior of $W_{j_0,t}^{(L+1)}(x_i) \boldsymbol{1} (\alpha_{j_0,t}^{(L)}>0)$ is infinity. 
 To work on this, we focus on the bias term in the $L$-th layer, i.e., $b^{(L)}$. 
 \begin{equation}
 \begin{aligned}
 K_t(x,x) &\geq \sum_{l=1}^L \langle \nabla_{b^{(l)}} f(x) , \nabla_{b^{(l)}} f(x) \rangle \geq \langle \nabla_{b^{(L)}} f(x) , \nabla_{b^{(L)}} f(x) \rangle
 \\& = \left\langle \left[ \sqrt{\frac{2}{m_{L}}} W^{(L+1)}D_x^{(L)} \right] , \left[ \sqrt{\frac{2}{m_{L}}} W^{(L+1)}D_x^{(L)} \right] \right\rangle 
 \\&\geq \frac{2}{m_L} \left( W_{j_0,t}^{(L+1)}(x_i) \boldsymbol{1} (\alpha_{j_0,t}^{(L)}>0) \right)^2 \to \infty. 
 \end{aligned}
 \end{equation}
It also contradicts to the assumption \eqref{eq: unif_conv_assumption_fc}. Therefore, we finish the proof. 
 


 \textbf{Part 2: Residual Network.}
 The proof is also accomplished through contradiction. We conduct it in a similar mannar as in the case of fully connected network. We first assume that there is a kind of parameter initialization such that 
 \begin{equation}\label{eq: unif_conv_assumption}
 \sup_{t\geq 0} \left \lvert K_t(x_i,x_j) - \RESNTK(x_i,x_j) \right \rvert \leq \frac{\lambda_0}{2n^2},
 \end{equation}
 holds for any $x_i ,x_j \in \mc{X}$. Proposition \ref{thm: positive_definite} ensures that $\lambda_0 >0$. Through Lemma \ref{lem: convergence_of_min_lambda} we have 
 \begin{equation}
 \inf_{t\geq 0} \widetilde{\lambda}_0(t) \geq \frac{\lambda_0}{2}>0.
 \end{equation}
 By Lemma \ref{lem: conv_of_function}, we have $(2y_i-1)f_t(x_i) \to +\infty$ as $t \to \infty$. 
 Recall the structure in \eqref{eq: resnet} of residual network, the network output function $f(x;\theta)$ can be actually decomposed as 
    \begin{equation}
    \begin{aligned}
        f(x;\theta) &= W^{(L+1)}\alpha^{(L)}(x);\\
        \alpha^{(L)}(x)& = \sqrt{\frac{1}{m_0}} (Ax + b) + \sum_{l=1}^L \sqrt{\frac{1}{m}} \alpha \left( V^{(l)} \sigma\left( \sqrt{\frac{2}{m}} W^{(l)} \alpha^{(l-1)}(x) \right) + d^{(l)}\right).
         \end{aligned}
    \end{equation}
 Since for any $x_i$ in training set, we have either $ \norm{W^{(L+1)}_t}_2 \to \infty$ or $ \norm{ \alpha^{(L)}_t(x_i)}_2 \to \infty$ in the sense of the limit superior. Namely, we have either 
 \begin{equation}
 \limsup_{t \to \infty} \norm{W_t^{(L+1)}}_2 = \infty, \quad \text{or} \quad \limsup_{t \to \infty} \norm{\alpha_t^{(L)}(x_i)}_2 = \infty. 
 \end{equation}

 If the former case holds, we have 
 \begin{equation}
 \left \langle \frac{\partial f_t(x_i)}{\partial d^{(L)}} , \frac{\partial f_t(x_i)}{\partial d^{(L)}} \right \rangle = \frac{1}{m} \norm{ W^{(L+1)}_t }_2^2 \to \infty, 
 \end{equation}
in the sense of the limit superior. And thus the empirical NTK will not converges to the NTK.

 If the latter case holds, we also have 
 \begin{equation}
 \left \langle \frac{\partial f_t(x_i)}{\partial W^{(L+1)}} , \frac{\partial f_t(x_i)}{\partial W^{(L+1)}} \right \rangle = \norm{ \alpha^{(L)}_t }_2^2 \to \infty, 
 \end{equation}
in the sense of the limit superior.

Therefore, the theorem is proved through contradiction. 
 
\end{proof}

\begin{lemma}\label{lem: convergence_of_min_lambda}
 Consider fully connected network \ref{eq: fc_network} and residual network \ref{eq: resnet}. Suppose the empirical NTK uniformly converges to the NTK over the samples, as the network width $m$ comes to infinity, namely,
 \begin{equation}
 \sup_{t\geq 0} \sup_{i,j\in[n]} |K_t(x_i,x_j) - K(x_i,x_j)| \leq \frac{\lambda_0}{2n^2},
 \end{equation}
 then we can postulate a lower bound of $\widetilde{\lambda}_0 $:
 \begin{equation}
 \inf_{t \geq 0} \widetilde{\lambda}_0(t) \geq \frac{\lambda_0}{2}.
 \end{equation}
 
\end{lemma}

\begin{proof}
 \begin{equation}
 \begin{aligned}
 |\widetilde{\lambda}_0(t) - \lambda_0| &\leq \norm{ K_t(X,X) - K(X,X) }_2 \leq \norm{ K_t(X,X) - K(X,X) }_F \\&\leq \sum_{i=1}^n \sum_{j=1}^n |K_t(x_i,x_j) - K(x_i,x_j)|\leq \frac{\lambda_0}{2}.
 \end{aligned}
 \end{equation}
 Thus the lemma is proved.
\end{proof}

\begin{lemma}\label{lem: conv_of_function}
 Consider fully connected network \ref{eq: fc_network} and residual network \ref{eq: resnet}. If we have a positive constant lower bound $C$ of $\widetilde{\lambda}_0(t)$ during training, then the network function will comes to infinity at the sample points $\{x_i\}_{i\in [n]}$, i.e.,
 \begin{equation}
 \lim_{t \to +\infty}(2y_i - 1)f_t(x_i) = + \infty.
 \end{equation}
\end{lemma}

\begin{proof}
Since $K^r = \diag(2Y-1) K_t(X,X) \diag(2Y-1)^{-1}$, we know $K^r$ share the same eigenvalues as $K_t(X,X)$. Therefore, we have $\lambda_{\min}(K^r) = \widetilde{\lambda}_0 (t) $. Define a function $V(r) = - \sum_{i=1}^n \ln (1-r_i)$, where $r = (r_1,r_2,\cdots,r_n)^\top \in [0,1)^n$. Apparently, we have $V(r) \geq 0$. Given the dynamic of $u$ \eqref{eq: dynamic_of_u}, we can derive the dynamic of $V(u)$:
\begin{equation} \label{eq: dynamic_of_vu}
\begin{aligned}
 \frac{\dd}{\dd t} V(u) &= \sum_{i=1}^n \frac{\partial V}{\partial r_i} \Big|_{r_i = u_i} \frac{\dd u_i }{\dd t} = \sum_{i=1}^n \frac{1}{1-u_i} \frac{\dd u_i}{\dd t}
 \\& = - \sum_{i=1}^n u_i [K_i^r]^\top u = -u^\top K^r u < 0.
\end{aligned}
\end{equation}
According to the monotone convergence theorem, we know $V_t = V(u)$ converges as $t \to \infty$. We can also prove that the limit of $V_t$ is exact zero. We give the proof by contradiction. 

First, we assume that $\lim_{t \to \infty} V_t = V_* > 0$ holds. Then for any $\varepsilon>0$, there exists $t_0 >0$, such that for any $t>t_0$, we have $V_*\leq V_t \leq V_* + \varepsilon$ holds. We define the following set:
\begin{equation}
\begin{aligned}
 \Gamma_\varepsilon & = [0,1-e^{-(V_*+\varepsilon)}]^n \setminus [0,1-e^{\frac{V_*}{n}})^n
 \\&= \Big\{r \in [0,1)^n \Big| 1- \me^{V_*/n} \leq \max_{i = 1,2,\cdots,n} r_i \leq 1-\me^{-(V_*+\varepsilon)} \Big\}.
\end{aligned}
\end{equation}
We can verify that $V^{-1}[V_*,V_* + \varepsilon] \subset \Gamma_\varepsilon$. Therefore, we have $u \in \Gamma_\varepsilon$ when $t>t_0$. Given that $\Gamma_\varepsilon$ is a compact set, we define $M = \min_{r \in \Gamma_\varepsilon} \norm{r}^2 >0$. When $r \in \Gamma_\varepsilon$, we can derive that 
\begin{equation}
 r^\top K^r r \geq \lambda_{\min}(K^r) \norm{r}^2 \geq \lambda_{\min}(K^r) M \geq MC>0,
\end{equation}
then there is a contradiction that
\begin{align*}
&\frac{\dd}{\dd t} V(u) = -u^\top K^r u \leq -MC
\\&\qquad \Longrightarrow V_t \leq V_{t_0} + \int_{t_0}^\top -MC \,\dd t = V_{t_0} -MC(t - t_0) \to - \infty.
\end{align*}
Thus, we prove that the limit of $V_t$ is zero. Based on the limit of $V_t$, we can also get the convergence result of $u_i$:
\begin{equation}
 \begin{aligned}
 & V(u) \to 0
 \\& \Longrightarrow 0<u_i \leq -\ln(1-u_i) \leq -\sum_{i=1}^n \ln(1-u_i) = V(u) \to 0.
 \end{aligned}
\end{equation}
Therefore, we have $\lim_{t \to +\infty}(2y_i - 1)f_t(x_i) = + \infty$ as the time $t$ comes to infinity holds for any $i \in [n]$. In this way, we finish the proof. 
\end{proof}

\begin{remark}\label{rem: general_loss}
The result of Lemma \ref{lem: conv_of_function} can be extended to more general loss functions. Specifically, the influence of the loss function is reflected in \eqref{eq: dynamic_of_vu}. Let \( l(x) \) be a convex loss function with \( l'(x) < 0 \), which appears in training as \( l\left( (2y_i - 1) f(x_i) \right) \). Under these conditions, the proof of Lemma \ref{lem: conv_of_function} still holds. We only need to define
\[
u \coloneqq \frac{\partial l\left( (2y_i -1) f(x_i) \right)}{\partial f(x_i)},
\]
and set \( V(u(\cdot)) = l(\cdot) \) as replacemment. This demonstrates that our proof is general under these assumptions.
\end{remark}

%% file: iclr2025_conference.bbl
\begin{thebibliography}{45}
\providecommand{\natexlab}[1]{#1}
\providecommand{\url}[1]{\texttt{#1}}
\expandafter\ifx\csname urlstyle\endcsname\relax
  \providecommand{\doi}[1]{doi: #1}\else
  \providecommand{\doi}{doi: \begingroup \urlstyle{rm}\Url}\fi

\bibitem[Allen-Zhu et~al.(2019)Allen-Zhu, Li, and Song]{allen2019convergence}
Zeyuan Allen-Zhu, Yuanzhi Li, and Zhao Song.
\newblock A convergence theory for deep learning via over-parameterization.
\newblock In \emph{International Conference on Machine Learning}, pp.\  242--252. PMLR, 2019.

\bibitem[Arora et~al.(2019)Arora, Du, Hu, Li, Salakhutdinov, and Wang]{arora2019exact}
Sanjeev Arora, Simon~S Du, Wei Hu, Zhiyuan Li, Russ~R Salakhutdinov, and Ruosong Wang.
\newblock On exact computation with an infinitely wide neural net.
\newblock \emph{Advances in Neural Information Processing Systems}, 32, 2019.

\bibitem[Bauer \& Kohler(2019)Bauer and Kohler]{bauer2019deep}
Benedikt Bauer and Michael Kohler.
\newblock {On deep learning as a remedy for the curse of dimensionality in nonparametric regression}.
\newblock \emph{The Annals of Statistics}, 47\penalty0 (4):\penalty0 2261 -- 2285, 2019.

\bibitem[Belfer et~al.(2024)Belfer, Geifman, Galun, and Basri]{belfer2024spectral}
Yuval Belfer, Amnon Geifman, Meirav Galun, and Ronen Basri.
\newblock Spectral analysis of the neural tangent kernel for deep residual networks.
\newblock \emph{Journal of Machine Learning Research}, 25\penalty0 (184):\penalty0 1--49, 2024.

\bibitem[Caponnetto \& De~Vito(2007)Caponnetto and De~Vito]{caponnetto2007optimal}
Andrea Caponnetto and Ernesto De~Vito.
\newblock Optimal rates for the regularized least-squares algorithm.
\newblock \emph{Foundations of Computational Mathematics}, 7:\penalty0 331--368, 2007.

\bibitem[Chen et~al.(2024)Chen, Li, and Lin]{chen2024impacts}
Guhan Chen, Yicheng Li, and Qian Lin.
\newblock On the impacts of the random initialization in the neural tangent kernel theory, 2024.
\newblock URL \url{https://arxiv.org/abs/2410.05626}.

\bibitem[Chizat et~al.(2019)Chizat, Oyallon, and Bach]{chizat2019lazy}
Lenaic Chizat, Edouard Oyallon, and Francis Bach.
\newblock On lazy training in differentiable programming.
\newblock \emph{Advances in neural information processing systems}, 32, 2019.

\bibitem[Cohen et~al.(2016)Cohen, Sharir, and Shashua]{cohen2016expressive}
Nadav Cohen, Or~Sharir, and Amnon Shashua.
\newblock On the expressive power of deep learning: A tensor analysis.
\newblock In \emph{Conference on learning theory}, pp.\  698--728. PMLR, 2016.

\bibitem[Cybenko(1989)]{cybenko1989approximation}
George Cybenko.
\newblock Approximation by superpositions of a sigmoidal function.
\newblock \emph{Mathematics of control, signals and systems}, 2\penalty0 (4):\penalty0 303--314, 1989.

\bibitem[Devlin et~al.(2019)Devlin, Chang, Lee, and Toutanova]{devlin2019bertpretrainingdeepbidirectional}
Jacob Devlin, Ming-Wei Chang, Kenton Lee, and Kristina Toutanova.
\newblock Bert: Pre-training of deep bidirectional transformers for language understanding, 2019.
\newblock URL \url{https://arxiv.org/abs/1810.04805}.

\bibitem[Dosovitskiy et~al.(2021)Dosovitskiy, Beyer, Kolesnikov, Weissenborn, Zhai, Unterthiner, Dehghani, Minderer, Heigold, Gelly, Uszkoreit, and Houlsby]{dosovitskiy2021imageworth16x16words}
Alexey Dosovitskiy, Lucas Beyer, Alexander Kolesnikov, Dirk Weissenborn, Xiaohua Zhai, Thomas Unterthiner, Mostafa Dehghani, Matthias Minderer, Georg Heigold, Sylvain Gelly, Jakob Uszkoreit, and Neil Houlsby.
\newblock An image is worth 16x16 words: Transformers for image recognition at scale, 2021.
\newblock URL \url{https://arxiv.org/abs/2010.11929}.

\bibitem[Du et~al.(2019{\natexlab{a}})Du, Lee, Li, Wang, and Zhai]{du2019gradient}
Simon Du, Jason Lee, Haochuan Li, Liwei Wang, and Xiyu Zhai.
\newblock Gradient descent finds global minima of deep neural networks.
\newblock In \emph{International conference on machine learning}, pp.\  1675--1685. PMLR, 2019{\natexlab{a}}.

\bibitem[Du et~al.(2019{\natexlab{b}})Du, Zhai, Poczos, and Singh]{du2018gradient}
Simon~S. Du, Xiyu Zhai, Barnabas Poczos, and Aarti Singh.
\newblock Gradient descent provably optimizes over-parameterized neural networks, 2019{\natexlab{b}}.
\newblock URL \url{https://arxiv.org/abs/1810.02054}.

\bibitem[Hanin \& Sellke(2018)Hanin and Sellke]{hanin2017approximating}
Boris Hanin and Mark Sellke.
\newblock Approximating continuous functions by relu nets of minimal width, 2018.
\newblock URL \url{https://arxiv.org/abs/1710.11278}.

\bibitem[He et~al.(2016)He, Zhang, Ren, and Sun]{he2016deep}
Kaiming He, Xiangyu Zhang, Shaoqing Ren, and Jian Sun.
\newblock Deep residual learning for image recognition.
\newblock In \emph{Proceedings of the IEEE conference on computer vision and pattern recognition}, pp.\  770--778, 2016.

\bibitem[Ho et~al.(2020)Ho, Jain, and Abbeel]{ho2020denoising}
Jonathan Ho, Ajay Jain, and Pieter Abbeel.
\newblock Denoising diffusion probabilistic models.
\newblock \emph{Advances in neural information processing systems}, 33:\penalty0 6840--6851, 2020.

\bibitem[Hornik et~al.(1989)Hornik, Stinchcombe, and White]{hornik1989multilayer}
Kurt Hornik, Maxwell Stinchcombe, and Halbert White.
\newblock Multilayer feedforward networks are universal approximators.
\newblock \emph{Neural networks}, 2\penalty0 (5):\penalty0 359--366, 1989.

\bibitem[Hron et~al.(2020)Hron, Bahri, Sohl-Dickstein, and Novak]{hron2020infinite}
Jiri Hron, Yasaman Bahri, Jascha Sohl-Dickstein, and Roman Novak.
\newblock Infinite attention: Nngp and ntk for deep attention networks.
\newblock In \emph{International Conference on Machine Learning}, pp.\  4376--4386. PMLR, 2020.

\bibitem[Hu et~al.(2021)Hu, Wang, Lin, and Cheng]{hu2021regularization}
Tianyang Hu, Wenjia Wang, Cong Lin, and Guang Cheng.
\newblock Regularization matters: A nonparametric perspective on overparametrized neural network.
\newblock In \emph{International Conference on Artificial Intelligence and Statistics}, pp.\  829--837. PMLR, 2021.

\bibitem[Huang et~al.(2020)Huang, Wang, Tao, and Zhao]{huang2020deep}
Kaixuan Huang, Yuqing Wang, Molei Tao, and Tuo Zhao.
\newblock Why do deep residual networks generalize better than deep feedforward networks?---a neural tangent kernel perspective.
\newblock \emph{Advances in neural information processing systems}, 33:\penalty0 2698--2709, 2020.

\bibitem[Jacot et~al.(2018)Jacot, Gabriel, and Hongler]{jacot2018neural}
Arthur Jacot, Franck Gabriel, and Cl{\'e}ment Hongler.
\newblock Neural tangent kernel: Convergence and generalization in neural networks.
\newblock \emph{Advances in neural information processing systems}, 31, 2018.

\bibitem[Ji \& Telgarsky(2020)Ji and Telgarsky]{ji2019polylogarithmic}
Ziwei Ji and Matus Telgarsky.
\newblock Polylogarithmic width suffices for gradient descent to achieve arbitrarily small test error with shallow relu networks, 2020.
\newblock URL \url{https://arxiv.org/abs/1909.12292}.

\bibitem[Karras et~al.(2019)Karras, Laine, and Aila]{Karras_2019_CVPR}
Tero Karras, Samuli Laine, and Timo Aila.
\newblock A style-based generator architecture for generative adversarial networks.
\newblock In \emph{Proceedings of the IEEE/CVF Conference on Computer Vision and Pattern Recognition (CVPR)}, June 2019.

\bibitem[Kingma \& Welling(2022)Kingma and Welling]{kingma2013auto}
Diederik~P Kingma and Max Welling.
\newblock Auto-encoding variational bayes, 2022.
\newblock URL \url{https://arxiv.org/abs/1312.6114}.

\bibitem[Krizhevsky et~al.(2012)Krizhevsky, Sutskever, and Hinton]{krizhevsky2012imagenet}
Alex Krizhevsky, Ilya Sutskever, and Geoffrey~E Hinton.
\newblock Imagenet classification with deep convolutional neural networks.
\newblock \emph{Advances in neural information processing systems}, 25, 2012.

\bibitem[Lai et~al.(2023{\natexlab{a}})Lai, Xu, Chen, and Lin]{lai2023generalizationabilitywideneural}
Jianfa Lai, Manyun Xu, Rui Chen, and Qian Lin.
\newblock Generalization ability of wide neural networks on $\mathbb{R}$, 2023{\natexlab{a}}.
\newblock URL \url{https://arxiv.org/abs/2302.05933}.

\bibitem[Lai et~al.(2023{\natexlab{b}})Lai, Yu, Tian, and Lin]{lai2023generalizationabilitywideresidual}
Jianfa Lai, Zixiong Yu, Songtao Tian, and Qian Lin.
\newblock Generalization ability of wide residual networks, 2023{\natexlab{b}}.
\newblock URL \url{https://arxiv.org/abs/2305.18506}.

\bibitem[Lai et~al.(2024)Lai, Li, Huang, and Lin]{lai2024optimalitykernelclassifierssobolev}
Jianfa Lai, Zhifan Li, Dongming Huang, and Qian Lin.
\newblock The optimality of kernel classifiers in sobolev space, 2024.
\newblock URL \url{https://arxiv.org/abs/2402.01148}.

\bibitem[LeCun et~al.(1998)LeCun, Bottou, Bengio, and Haffner]{lecun1998gradient}
Yann LeCun, L{\'e}on Bottou, Yoshua Bengio, and Patrick Haffner.
\newblock Gradient-based learning applied to document recognition.
\newblock \emph{Proceedings of the IEEE}, 86\penalty0 (11):\penalty0 2278--2324, 1998.

\bibitem[Lee et~al.(2019)Lee, Xiao, Schoenholz, Bahri, Novak, Sohl-Dickstein, and Pennington]{lee2019wide}
Jaehoon Lee, Lechao Xiao, Samuel Schoenholz, Yasaman Bahri, Roman Novak, Jascha Sohl-Dickstein, and Jeffrey Pennington.
\newblock Wide neural networks of any depth evolve as linear models under gradient descent.
\newblock \emph{Advances in neural information processing systems}, 32, 2019.

\bibitem[Li \& Lin(2025)Li and Lin]{li2025diagonaloverparameterizationreproducingkernel}
Yicheng Li and Qian Lin.
\newblock Diagonal over-parameterization in reproducing kernel hilbert spaces as an adaptive feature model: Generalization and adaptivity, 2025.
\newblock URL \url{https://arxiv.org/abs/2501.08679}.

\bibitem[Li et~al.(2024)Li, Yu, Chen, and Lin]{li2024eigenvalue}
Yicheng Li, Zixiong Yu, Guhan Chen, and Qian Lin.
\newblock On the eigenvalue decay rates of a class of neural-network related kernel functions defined on general domains.
\newblock \emph{Journal of Machine Learning Research}, 25\penalty0 (82):\penalty0 1--47, 2024.
\newblock URL \url{http://jmlr.org/papers/v25/23-0866.html}.

\bibitem[Lin et~al.(2020)Lin, Rudi, Rosasco, and Cevher]{lin2020optimal}
Junhong Lin, Alessandro Rudi, Lorenzo Rosasco, and Volkan Cevher.
\newblock Optimal rates for spectral algorithms with least-squares regression over {{Hilbert}} spaces.
\newblock \emph{Applied and Computational Harmonic Analysis}, 48\penalty0 (3):\penalty0 868--890, May 2020.
\newblock ISSN 1063-5203.
\newblock \doi{10.1016/j.acha.2018.09.009}.

\bibitem[Liu et~al.(2020)Liu, Zhu, and Belkin]{liu2020linearity}
Chaoyue Liu, Libin Zhu, and Misha Belkin.
\newblock On the linearity of large non-linear models: when and why the tangent kernel is constant.
\newblock \emph{Advances in Neural Information Processing Systems}, 33:\penalty0 15954--15964, 2020.

\bibitem[Lu et~al.(2017)Lu, Pu, Wang, Hu, and Wang]{lu2017expressive}
Zhou Lu, Hongming Pu, Feicheng Wang, Zhiqiang Hu, and Liwei Wang.
\newblock The expressive power of neural networks: A view from the width.
\newblock \emph{Advances in neural information processing systems}, 30, 2017.

\bibitem[Nguyen et~al.(2021)Nguyen, Mondelli, and Montufar]{nguyen2021tight}
Quynh Nguyen, Marco Mondelli, and Guido~F Montufar.
\newblock Tight bounds on the smallest eigenvalue of the neural tangent kernel for deep relu networks.
\newblock In \emph{International Conference on Machine Learning}, pp.\  8119--8129. PMLR, 2021.

\bibitem[Schmidt-Hieber(2020)]{schmidt2020nonparametric}
Johannes Schmidt-Hieber.
\newblock {Nonparametric regression using deep neural networks with ReLU activation function}.
\newblock \emph{The Annals of Statistics}, 48\penalty0 (4):\penalty0 1875 -- 1897, 2020.
\newblock \doi{10.1214/19-AOS1875}.
\newblock URL \url{https://doi.org/10.1214/19-AOS1875}.

\bibitem[Steinwart \& Christmann(2008)Steinwart and Christmann]{steinwart2008support}
Ingo Steinwart and Andreas Christmann.
\newblock \emph{Support vector machines}.
\newblock Springer Science \& Business Media, 2008.

\bibitem[Suh et~al.(2021)Suh, Ko, and Huo]{suh2021non}
Namjoon Suh, Hyunouk Ko, and Xiaoming Huo.
\newblock A non-parametric regression viewpoint: Generalization of overparametrized deep relu network under noisy observations.
\newblock In \emph{International Conference on Learning Representations}, 2021.

\bibitem[Taheri \& Thrampoulidis(2024)Taheri and Thrampoulidis]{taheri2024generalization}
Hossein Taheri and Christos Thrampoulidis.
\newblock Generalization and stability of interpolating neural networks with minimal width.
\newblock \emph{Journal of Machine Learning Research}, 25\penalty0 (156):\penalty0 1--41, 2024.

\bibitem[Tian \& Yu(2024)Tian and Yu]{tian2024improve}
Songtao Tian and Zixiong Yu.
\newblock Improve generalization ability of deep wide residual network with a suitable scaling factor, 2024.
\newblock URL \url{https://arxiv.org/abs/2403.04545}.

\bibitem[Vaswani et~al.(2017)Vaswani, Shazeer, Parmar, Uszkoreit, Jones, Gomez, Kaiser, and Polosukhin]{vaswani2017attention}
Ashish Vaswani, Noam Shazeer, Niki Parmar, Jakob Uszkoreit, Llion Jones, Aidan~N Gomez, Lukasz Kaiser, and Illia Polosukhin.
\newblock Attention is all you need.
\newblock \emph{Advances in neural information processing systems}, 30, 2017.

\bibitem[Vyas et~al.(2022)Vyas, Bansal, and Nakkiran]{vyas2022limitations}
Nikhil Vyas, Yamini Bansal, and Preetum Nakkiran.
\newblock Limitations of the ntk for understanding generalization in deep learning, 2022.
\newblock URL \url{https://arxiv.org/abs/2206.10012}.

\bibitem[Wang et~al.(2017)Wang, Jiang, Qian, Yang, Li, Zhang, Wang, and Tang]{Wang_2017_CVPR}
Fei Wang, Mengqing Jiang, Chen Qian, Shuo Yang, Cheng Li, Honggang Zhang, Xiaogang Wang, and Xiaoou Tang.
\newblock Residual attention network for image classification.
\newblock In \emph{Proceedings of the IEEE Conference on Computer Vision and Pattern Recognition (CVPR)}, July 2017.

\bibitem[Zhang et~al.(2024)Zhang, Chen, Tian, and Lu]{zhang2024unified}
Shao-Qun Zhang, Zong-Yi Chen, Yong-Ming Tian, and Xun Lu.
\newblock A unified kernel for neural network learning, 2024.
\newblock URL \url{https://arxiv.org/abs/2403.17467}.

\end{thebibliography}
